\documentclass[final,12pt]{colt2021}
\makeatletter
\def\set@curr@file#1{\def\@curr@file{#1}} 
\makeatother
\usepackage[load-configurations=version-1]{siunitx} 

\usepackage{fancyhdr}
\editor{}
\jmlrvolume{}
\jmlryear{}
\jmlrsubmitted{}
\jmlrpublished{}
\jmlrworkshop{}
\titletag{}
\jmlrproceedings{}{}
\jmlrpages{}

\usepackage{times}
\usepackage[utf8]{inputenc}
\usepackage[T1]{fontenc}   
\usepackage{hyperref}      
\usepackage{url}           
\usepackage{booktabs}      
\usepackage{nicefrac}      
\usepackage{microtype}     
\usepackage{amsfonts}
\usepackage{algorithm2e}
\usepackage{bm}
\usepackage{enumitem}
\usepackage{graphicx}
\LinesNumbered 
\DontPrintSemicolon

\usepackage{tikz}
\usetikzlibrary{backgrounds,positioning,shapes,calc,intersections,through}
\usetikzlibrary{quotes,angles}
\tikzstyle{dot}=[draw,fill,shape=circle,inner sep=0pt,minimum size=3pt]
\tikzstyle{ps}=[circle,draw, fill=black, minimum size=4,inner sep=0pt, outer sep=0pt]
\tikzstyle{ns}=[circle,draw, fill=white, minimum size=4,inner sep=0pt, outer sep=0pt]
\tikzstyle{mve}=[circle,draw,densely dotted,thick]
\tikzstyle{ell}=[circle,draw,thick]
\usepackage{pgfplots}

\newcommand{\algoname}[1]{\ifmmode\operatorname{#1}\else{\normalfont #1}\fi}

\newcommand{\RecoverClustering}{\algoname{RecoverClustering}}

\newcommand{\GetEpsilons}{\algoname{GetEpsilons}}
\newcommand{\GetEpsilon}{\algoname{GetEpsilon}}
\newcommand{\MST}{\algoname{MST}}
\newcommand{\IsConnected}{\algoname{IsConnected}}
\newcommand{\Cut}{\algoname{ClusterSeparator}}

\newcommand{\BallSeparator}{\algoname{MBS}}
\newcommand{\FindNewSeed}{\algoname{FindNewSeed}}
\newcommand{\RecoverSingleCluster}{\algoname{RecoverSingleCluster}}
\newcommand{\FindCutEdge}{\algoname{FindCutEdge}}
\newcommand{\ReduceRecover}{\algoname{RecoverClustering2}}
\newcommand{\edgecut}{\Gamma}

\newcommand{\dg}{d_G}
\newcommand{\scG}{\mathcal{G}}

\newcommand{\field}[1]{\mathbb{#1}}
\newcommand{\R}{\field{R}}
\newcommand{\Nat}{\field{N}}

\newcommand{\C}{\mathcal{C}}
\newcommand{\scE}{\mathcal{E}}

\newcommand{\scO}{\mathcal{O}}

\newcommand{\clus}{\C}

\newcommand{\nil}{\textsc{nil}}
\newcommand{\neps}{\epsilon_{\downarrow}}
\newcommand{\theset}[2]{ \left\{ {#1} \,:\, {#2} \right\} }
\newcommand{\Ind}[1]{ \field{I}\left\{{#1}\right\} }

\newcommand{\ve}{\varepsilon}

\newcommand{\bs}{\pmb{s}}
\newcommand{\bu}{\pmb{u}}
\newcommand{\beps}{\pmb{\epsilon}}
\newcommand{\dens}{\operatorname{dens}}
\newcommand{\CoverNum}{\mathcal{N}} 
\newcommand{\PackNum}{\mathcal{M}} 

\renewcommand{\hat}{\widehat}

\renewcommand{\epsilon}{\ve}

\newtheorem{observation}{Observation}
\newtheorem{claim}{Claim}

\DeclareMathOperator{\poly}{poly}
\newcommand{\scq}{\textsc{scq}}
\newcommand{\seed}{\textsc{seed}}

\newcommand{\rad}{\epsilon}

\newcommand{\doubl}{\operatorname{dbl}}

\pgfplotsset{compat=1.16}
\def\mytitle{Exact Recovery of Clusters in Finite Metric Spaces \linebreak Using Oracle Queries}
\title[\mytitle]{\mytitle}
\usepackage{times}
\coltauthor{%
 \Name{Marco Bressan} \Email{marco.bressan@unimi.it}\\
 \addr University of Milan
 \AND
 \Name{Nicolò Cesa-Bianchi} \Email{nicolo.cesa-bianchi@unimi.it}\\
 \addr University of Milan
 \AND
 \Name{Silvio Lattanzi} \Email{silviol@google.com}\\
 \addr Google
 \AND
 \Name{Andrea Paudice} \Email{andrea.paudice@unimi.it}\\
 \addr University of Milan \& Italian Institute of Technology
}

\usepackage{thmtools, thm-restate}

\begin{document}

\maketitle

\begin{abstract}%
We investigate the problem of exact cluster recovery using oracle queries.
Previous results show that clusters in Euclidean spaces that are convex and separated with a margin can be reconstructed exactly using only $\scO(\log n)$ same-cluster queries, where $n$ is the number of input points.
In this work, we study this problem in the more challenging non-convex setting. We introduce a structural characterization of clusters, called $(\beta,\gamma)$-convexity, that can be applied to any finite set of points equipped with a metric (or even a semimetric, as the triangle inequality is not needed).
Using $(\beta,\gamma)$-convexity, we can translate natural density properties of clusters (which include, for instance, clusters that are strongly non-convex in $\R^d$) into a graph-theoretic notion of convexity.
By exploiting this convexity notion, we design a deterministic algorithm that recovers $(\beta,\gamma)$-convex clusters using $\scO(k^2 \log n + k^2 (\nicefrac{6}{\beta\gamma})^{\dens(X)})$ same-cluster queries, where $k$ is the number of clusters and $\dens(X)$ is the density dimension of the semimetric.
We show that an exponential dependence on the density dimension is necessary, and we also show that, if we are allowed to make $\scO(k^2 + k\log n)$ additional queries to a ``cluster separation'' oracle, then we can recover clusters that have different and arbitrary scales, even when the scale of each cluster is unknown.
\end{abstract}


\section{Introduction}
\label{sec:intro}
We investigate the problem of exact reconstruction of clusters using oracle queries in the semi-supervised active clustering framework (SSAC) of \citet{ashtiani2016clustering}.
In SSAC, we are given $n$ points in a metric space, and the goal is to partition these points into $k$ clusters with the help of an oracle answering queries of the form ``do $x$ and $y$ belong to the same cluster?''.
When the metric is Euclidean, \citet{ashtiani2016clustering}, \citet{DBLP:conf/nips/0002CLP20} and~\citet{BCLP21margin} show that exact reconstruction is possible using only $\scO(\log n)$ oracle queries, which mirrors the query complexity of efficient active learning.
These results heavily rely on the Euclidean geometry of the clusters; in particular, clusters are assumed to be convex (e.g., ellipsoidal) and separable with a margin.
These assumptions exclude many natural definitions of ``cluster'', such as those based on notions of density, or those computed by popular techniques like spectral clustering, linkage clustering, or DBSCAN.

In this work we study exact cluster recovery in metric spaces, or ---more generally--- finite semimetric spaces (where the triangle inequality is not necessarily satisfied).
The use of semimetrics in clustering, and in other machine learning tasks, is motivated by the fact that in many applications domains the notion of distance is strongly non-Euclidean \citep{gottlieb2017nearly}. Classic examples include the Wasserstein distance in vision, the Levenshtein distance in string matching, the cosine dissimilarity in document analysis, the Pearson dissimilarity in bioinformatics. In all these cases, the notions of convexity and separability with margin are lost, or exist only in certain generalized forms, so the cluster recovery techniques of \citet{ashtiani2016clustering,DBLP:conf/nips/0002CLP20,BCLP21margin} do not apply anymore. To fill this gap, we introduce a novel notion of cluster convexity that can be applied to any finite semimetric space and that can be exploited algorithmically.

We start by considering \emph{geodesic convexity} in weighted graphs \citep{pelayo2013geodesic}, a well-known generalization of Euclidean convexity that has been used, among others, for node classification in graphs~\citet{mlg2020_40}. Given a weighted graph $\scG$, a subset $C \subseteq V(\scG)$ is said to be geodesically convex if every shortest path between any two vertices of $C$ lies entirely in $C$. Thus, in a finite semimetric space, we could say that $C$ is a convex cluster if it is geodesically convex in the weighted graph $\scG$ encoding the semimetric (with the distances as weights).
Unfortunately, this condition is too lax.
To see this, take $n$ distinct points on a circle with the Euclidean metric.
In $\scG$, the shortest path between any two points $x,y$ is always the edge $(x,y)$ itself.
Thus, according to this definition, any subset of the $n$ points will be geodesically convex, and so every clustering will be admitted, which means that $\Omega(n)$ queries will be needed to recover the clustering.
However, we will show that a variant of this approach gives a suitable notion of convexity, one that yields efficient recovery with only $\scO(\log n)$ queries while capturing the density of the clusters.

\subsection*{Our contributions.}
We introduce \emph{$(\beta,\gamma)$-convex clusterings}, a novel family of clusterings defined on the weighted graph $\scG$ encoding the semimetric on $X$. For any $\epsilon > 0$, let $G_X(\epsilon)$ be the unweighted subgraph of $\scG$ obtained by deleting all edges $(x,y)$ with $d(x,y) > \epsilon$.
We say that a clustering is $(\beta,\gamma)$-convex (with $\beta,\gamma\in (0,1]$) if for some $\epsilon > 0$ every cluster $C$ satisfies the following three properties.
\emph{Connectivity}: the subgraph of $G_X(\epsilon)$ induced by $C$ is connected.
\emph{Local metric margin}: if $x \in C$ and $y \notin C$, then $d(x,y) > \beta \epsilon$\footnote{Note that, for any clustering $C$ in a finite semimetric space $X$ and for any $\epsilon > 0$, a $\beta$ such that the local metric margin condition holds can always be found. In this respect, $\beta$ defines a hierarchy over clusterings, where large values of $\beta$ identify clusterings that are easier to learn.
}.
\emph{Geodesic convexity with margin}: in $G_X(\epsilon)$, if $x,y \in C$ and the shortest path between $x$ and $y$ has length $\ell$, then any simple path between $x$ and $y$ of length at most $\ell (1+\gamma)$ does not leave $C$.
The smallest $\epsilon$ for which these properties hold is called the \emph{radius} of the clusters.
It is important to observe that $(\beta,\gamma)$-convexity includes nontrivial cases. 
For instance, $(\beta,\gamma)$-convex clusters can be strongly non-convex in $\R^d$, as depicted in Figure~\ref{fig:whirl}. Moreover, we can allow the clusters to have different radii $\epsilon_1,\ldots,\epsilon_k$ (see below), as depicted in Figure~\ref{fig:oort}.
These examples suggest that, in a certain sense, one can view $(\beta,\gamma)$-convexity as a way of translating density into convexity.
Moreover, if we drop any one of the three conditions ---connectedness, local metric margin, and geodesic convexity with margin--- the clusters can become disconnected, too close to one another, or interspersed with other clusters, see Section~\ref{sec:pathological}.
\begin{figure}[h!]
    \centering
    \includegraphics[scale=.45]{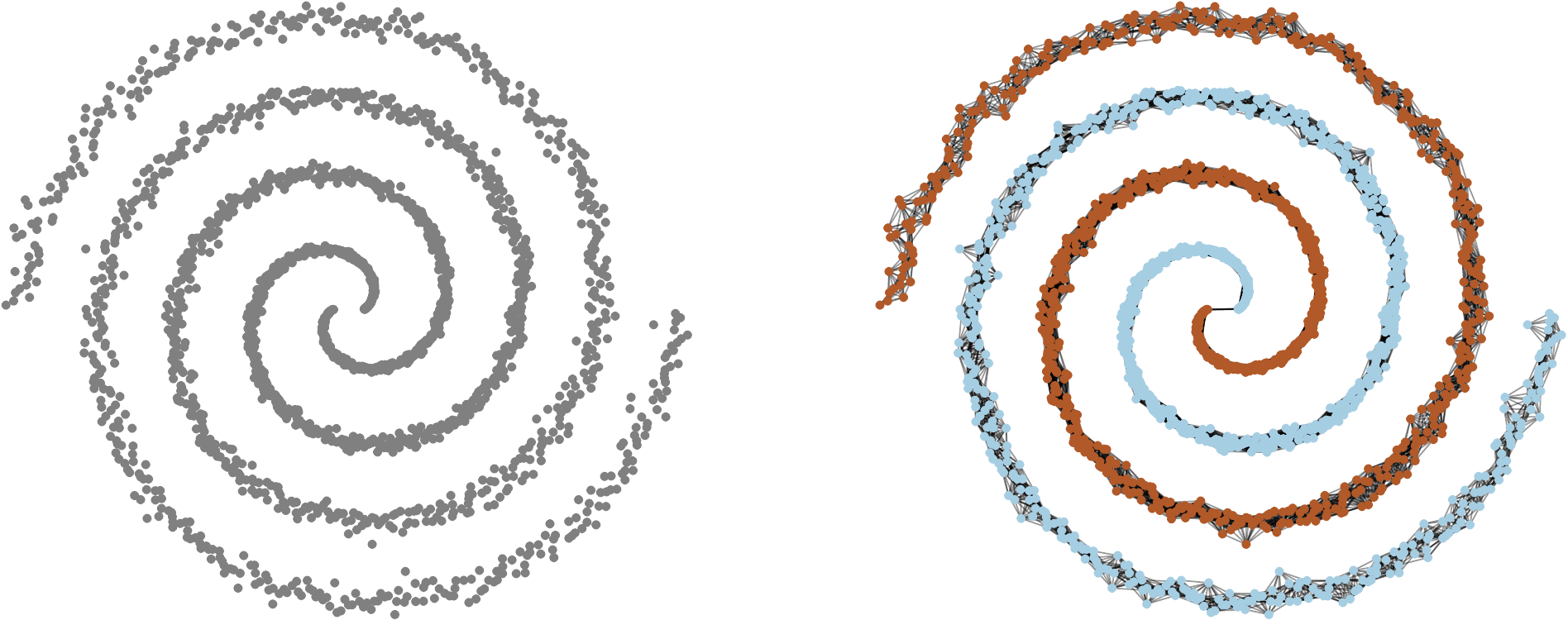}
    \caption{Left: a toy point set. Right: the graph $G_X(\epsilon)$ and a valid $(\beta,\gamma)$-convex clustering for $X$ (clusters encoded by the color of the points), for any $\beta < \frac{1}{2}$ and $\gamma > 0$.
    }
    \label{fig:whirl}
    \includegraphics[scale=.49]{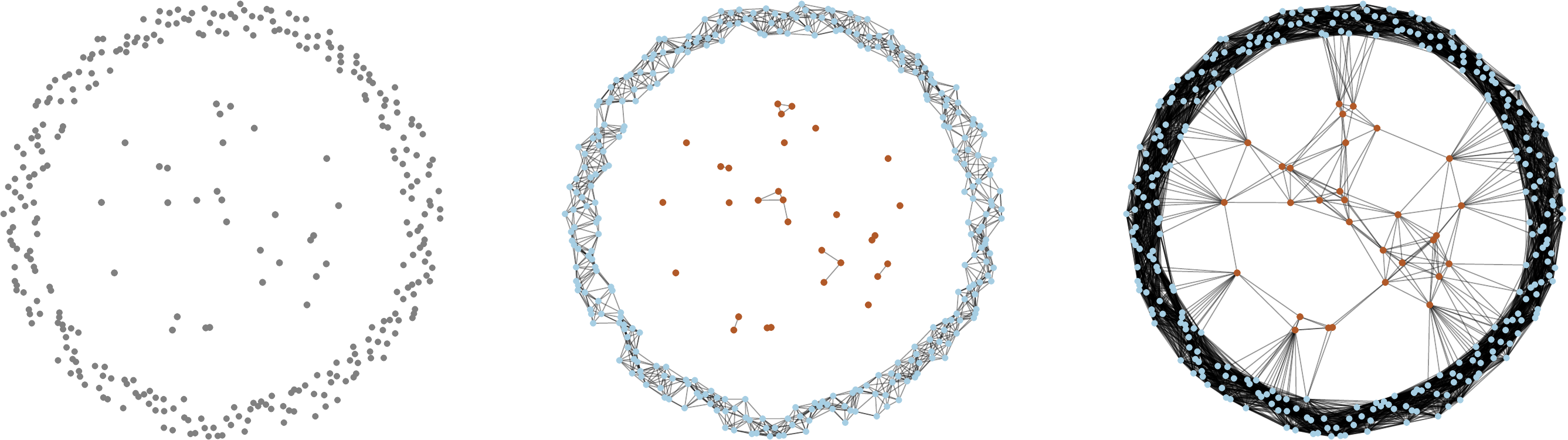}
    \caption{Left: a toy point set. Middle and right: the graphs $G_X(\rad_1)$ and $G_X(\rad_2)$ where $\epsilon_1<\epsilon_2$ are the radii of the ``outer'' cluster $C_1$ and the ``inner'' cluster $C_2$.
    The clustering is $(\beta,\gamma)$-convex for $\beta \le .5$ and $\gamma \le .1$.
    }
    \label{fig:oort}
\end{figure}

Our first result shows that $(\beta,\gamma)$-convex clusterings can be recovered efficiently using a small number of same-cluster queries (\scq\ for short).
More precisely, if $\epsilon,\beta,\gamma$ are known, and for each cluster we know an arbitrary initial point, called \emph{seed},\footnote{We note here that this last assumptions can be dropped if the clusters have roughly similar sizes. In fact, if the gap between the size of the largest and smallest clusters is $\chi$ we can obtain the seed w.h.p.\ by sampling $\tilde{\scO}(k \chi)$ nodes and retrieving their cluster membership using $\tilde{O}(k^2 \chi)$ same cluster queries} then we can deterministically recover all clusters with $\scO\big(k^2\log n + k^2\big(\nicefrac{6}{\beta\gamma}\big)^{\dens(X)}\big)$ \scq\ queries.
Here, $\dens(X)$ is the \emph{density dimension} of $X$ \citep{gottlieb2013proximity}, a generalization of the doubling dimension that, in metric spaces, is used to bound the size of packings.
This dependence of our exponent on $\dens(X)$ is asymptotically optimal, as we prove that, in the worst case, any algorithm needs $\Omega(2^{\dens(X)})$ \scq\ queries to recover a $(\beta,\gamma)$-convex clustering.
The running time of our algorithm is $\scO\big(k^2(n+m)\big)$, where $m$ is the number of edges of $\scG$ (i.e., the number of finite distances between the input points).
The key step of our algorithm consists in finding a \emph{cluster separator} between each cluster $C$ and the other clusters.
To do this, we need to find edges between $C$ and other clusters in the graph $G_X(\epsilon)$, which requires to carefully exploit the structural properties of $(\beta,\gamma)$-convexity.
We note that, interestingly, in the $r$-dimensional Euclidean setting, the cluster recovery algorithm of~\citep{BCLP21margin} makes a number of queries roughly of the order of $\scO\big((\nicefrac{1}{\gamma})^{r} \log n\big)$, where $\gamma$ is the convex margin.
This shows that our notion of convexity plays a role similar to that of Euclidean convexity.

Next, we investigate the power of queries.
First, we show that, without seed nodes, any algorithm using only the \scq\ oracle needs $\Omega(n)$ queries to recover the clusters. To circumvent this lower bound, we add a more powerful query, called \seed. Given a partition of $X$ and the id of a cluster, the \seed\ query provides a certificate (i.e., a point of $X$) that the cluster is cut by the partition, or answers negatively if the cluster is not cut. We show that, if we can use $\scO\big(k^2\log n + k^2\big(\nicefrac{6}{\beta\gamma}\big)^{\dens(X)}\big)$ \scq\ queries plus only $\scO(k^2)$ \seed\ queries, then we can recover clusters with \emph{different} radii $\epsilon_1,\ldots,\epsilon_k$, a case that we model by generalizing the $(\beta,\gamma)$-convex definition in a natural way. This allows us to capture clusters with different ``scales'', as shown in Figure~\ref{fig:oort}. If the radii $\epsilon_1,\ldots,\epsilon_k$ are unknown, we show that $\scO(k\log n)$ \seed\ queries are sufficient to learn them in time $\scO\big(m\alpha(m,n) + kn\log n\big)$, where $\alpha$ is the inverse of the Ackermann function, and that no algorithm can learn the radii with less than $\Omega(k \log \frac{n}{k})$ queries. Furthermore, if one of $\beta$ and $\gamma$ is unknown, we show that we can still learn the clusters by paying a small overhead.

\section{Related work}
Exact reconstruction of clusters with same-cluster queries was introduced by \citet{ashtiani2016clustering}, who showed how to recover exactly the optimal $k$-means clustering with $\scO(\poly(k) \log n)$ queries when each cluster lies inside a sphere centered in the cluster's centroid and well separated from the spheres of other clusters. \citet{DBLP:conf/nips/0002CLP20} extend these results to clusters separated by arbitrary ellipsoids with arbitrary centers, and~\citet{BCLP21margin} to arbitrary convex clusters with margin. Both results assume the standard Euclidean metric.

\seed\ queries have been used by~\citet{HannekePHD} as ``positive example queries'', by~\citet{BalcanHanneke12} as ``conditional class queries'', by~\citet{Beygelzimer16-seedqueries,Attenberg2010} as ``search queries'', and, implicitly, also by~\citet{Tong01-mistake-queries,Doyle2011}. Previous works also show that \seed\ queries are well justified in practice, as noted by~\citet{Beygelzimer16-seedqueries}.

As with $O(k)$ \scq\ queries one can learn the cluster id of any point (up to a relabeling of the clusters), we could reduce our problem to the problem of classifying the nodes of $G_X(\epsilon)$.
\citet{DasarathyS2} develop a probabilistic active classification algorithm, called $S^2$ (\emph{s}hortest-\emph{s}hortest-path), whose label complexity depends on the graph's structure. In particular, the label complexity is linear in the size of the boundary of the cut between nodes with different labels. Unfortunately, even under $(\beta,\gamma)$-convexity, in $G_X(\epsilon)$ this boundary can have size $\Omega(n)$. \citet{mlg2020_40} show a deterministic algorithm with label complexity proportional to the size of the shortest path cover of the graph (the smallest set of shortest paths that cover all nodes). Again, in $G_X(\epsilon)$ this cover could have size $\Omega(n)$ even under $(\beta,\gamma)$-convexity. Active classification on unweighted graph has been also studied by \citet{afshani2007complexity,cesa2010active,guillory2011active}, but only with approximate reconstruction guarantees (i.e., $\Omega(n)$ queries may be needed for exact recovery).

\citet{NIPS2017_7054} study exact cluster reconstruction with a \scq\ oracle on weighted graphs, where weights express similarities. They prove a logarithmic query bound using a Monte-Carlo algorithm and a log-linear bound using a Las-Vegas algorithm. However, similarly to a stochastic block model, they assume that the weights are drawn from some latent distribution that depends on the clustering. Stochastic block models \citep{zhang2014phase,gadde2016active} and geometric block models \citep{chien2020active} have been also considered as generative models for active clustering on graphs.

Center-based \citep{balcan2013active}, density-based \citep{mai2013active}, spectral \citep{wang2010active,shamir2011spectral}, and hierarchical \citep{eriksson2011active,krishnamurthy2012efficient} clustering have been also studied in a more restricted active learning setting, where the algorithm has access to the pairwise distances through an oracle.

\section{Preliminaries and notation}
\label{sec:prelim}
Our algorithms receive in input a semimetric represented by an undirected weighted graph $\scG=(X,\scE,d)$, where $d(x,y) > 0$ is the weight\footnote{Our query bounds do not depend on the size of the weights.} of the edge $(x,y) \in \scE$ and $|X|=n$.
For $\epsilon>0$, we let $G_X(\epsilon)$ be the undirected graph with vertex set $X$ where $x,y \in X$ are connected if and only if $d(x,y) \le \epsilon$.
Given a graph $G=(V,E)$ and two vertices $x,y \in V$, we denote by $\dg(x,y)$ the shortest-path distance between $x$ and $y$ in $G$ and by $\rho(G)$ the number of connected components of $G$.
Given any subset $U \subseteq V$, we use $G[U]$ to denote the subgraph of $G$ induced by $U$, and we use $\edgecut(U)$ to denote the set of edges having exactly one endpoint in $U$.
An edge $(x,y) \in \edgecut(U)$ is called a \emph{cut edge} of $U$.

For any $x \in X$ and $r > 0$, the ball of center $x$ and radius $r$ is $B(x,r) = \theset{y \in X}{d(x,y) \le r}$.
For any set $K \subseteq X$, we denote by $\PackNum(K,r)$ the maximum cardinality of any subset $A$ of $K$ such that all distinct $x,y \in A$ satisfy $d(x,y) > r$.
Following \citet{gottlieb2017nearly}, we define the \emph{density constant} of $X$ as:\footnote{While \citep{gottlieb2017nearly} uses open balls, we use closed balls.}
\begin{align}
    \mu(X) = \min\left\{ \mu \in \Nat \,:\, (x \in X) \wedge (r > 0) \Rightarrow \PackNum\!\left(B(x,r), \nicefrac{r}{2}\right) \le \mu \right\} \label{eq:muX}
\end{align}
Therefore, in $X$, any ball of radius $r$ contains at most $\mu(X)$ points at pairwise distance larger than $\frac{r}{2}$.
The \emph{density dimension} of $X$ is $\dens(X)=\log_2 \mu(X)$.
It is easy to see that, for any ball $B(x,r)$ and for any $\eta \in (0, 1)$ we have:
\begin{align}
    \PackNum(B(x,r), \eta r) \le \mu(X)^{\big\lceil \log_2 \frac{1}{\eta} \big\rceil} \le \left(\nicefrac{2}{\eta}\right)^{\dens(X)}
\label{eq:packnum}
\end{align}
When $d$ is a metric, we have $\doubl(X) \le \dens(X) \le 2 \doubl(X)$ where $\doubl(X)$ is the doubling dimension of $X$, see Lemma~\ref{lem:doubling} in Appendix~\ref{app:prelim}.
We denote by $\PackNum^*(\eta)$ the maximum of $\PackNum(B(x,r), \eta r)$ over all $x \in X$ and $r > 0$.
The quantity $\PackNum^*(\eta)$ appears in our bounds, with $\eta$ being a function of $\beta$ and $\gamma$.
Note that $\PackNum^*(\eta) \le \big(\nicefrac{2}{\eta}\big)^{\dens(X)}$, by~\eqref{eq:packnum}.

A $k$-clustering of $X$ is a partition $\clus=(C_1,\ldots,C_k)$ of $X$.
We denote by $\clus(x)$ the cluster id of $x$, so that $x \in C_{\clus(x)}$.
We now introduce the characterization of the clusterings considered in this work.
The following definition is for clusters with identical radii; a more complex version will be needed in the case with different radii, see Section~\ref{sec:general}.
\begin{definition}[$(\beta,\gamma)$-convex clustering]
\label{def:dense}
For any $\beta,\gamma \in (0,1]$, a $k$-clustering $\clus=(C_1,\ldots,C_k)$ of $X$ is $(\beta,\gamma)$-convex if $\,\exists\, \epsilon > 0$ such that for each $i \in [k]$ the following properties hold:
\begin{enumerate}[leftmargin=17pt,topsep=1pt,parsep=2pt,itemsep=0pt]
    \item \emph{connectedness}: the subgraph induced by $C_i$ in $G_X(\epsilon)$ is connected 
    \item \emph{local metric margin}: for all $x,y \in X$, if $x \in C_i$ and $y \notin C_i$, then $d(x,y) > \beta\,\! \epsilon$
    \item \emph{geodesic convexity with margin}: if $x,y \in C_i$, then in $G_X(\epsilon)$ any simple path between $x$ and $y$ of length at most $(1+\gamma) d_{G_X(\epsilon)}(x,y)$ lies entirely in $C_i$ 
\end{enumerate}
\end{definition}
The smallest value of $\epsilon$ satisfying the three properties is called the \emph{radius} of the clusters.\footnote{Actually, our algorithm of Section~\ref{sec:same} works with \emph{any} such $\epsilon$. We use the minimum to disambiguate the radius.}

In this work, we assume the algorithm obtains information about the unknown target clustering $\clus$ through \emph{same-cluster queries} \citep{ashtiani2016clustering}:
\begin{itemize}
\item[\scq:] for any $x,y\in X$, $\scq(x,y)$ returns $\Ind{\clus(x)=\clus(y)}$, where $\Ind{\cdot}$ is the indicator function
\end{itemize}
If only same-cluster queries are available, we assume that, together with $X$, the recovery algorithm is given a \emph{seed node} $s_i \in C_i$ for each cluster $C_i$. Seed nodes are collected in a vector $\bs=(s_1,\ldots,s_k)$.
Without seed nodes, $\Omega(n)$ same-cluster queries are needed to recover $\clus$ (see Section~\ref{sec:lb})\footnote{As we already said, if the sizes of the clusters are ``roughly'' balanced we can remove this assumption and sample $\widetilde{\scO}(k)$ random nodes and use $\widetilde{\scO}(k^2)$ same-cluster queries to obtain the seed nodes with high probability.}.

When the cluster radii are different and/or unknown, or the seed nodes are not available, we show that there are instances where $\Omega(n)$ same cluster queries are needed. To overcome this limitation, we allow the algorithm to use another type of queries, called \emph{seed queries}:
\begin{itemize}[leftmargin=32pt]
\item[\seed:] for any $S \subseteq X$ and $i\in [k]$, $\seed(S,i)$ returns an arbitrary point $s_i \in C_i \cap S$, or \nil\ if $C_i \cap S = \emptyset$.
\end{itemize}
The \seed\ query is a kind of separation oracle: given a partition $(S,X\setminus S)$ of $X$ and a cluster label $i$, the query can be used to check whether $C_i$ is cut by this partition. Indeed, if $C_i$ is cut, then \seed$(S,i)$ and \seed$(X \setminus S, i)$ return a point of $C_i$ belonging to, respectively, $S$ and $X\setminus S$. Note that these queries are very natural. In a crowdsourcing setting, they correspond to asking the rater to identify an entity (say a picture of a car) within a set (say a set of pictures).

\subsection{Necessity of the properties of Definition~\ref{def:dense}}
\label{sec:pathological}
We give some representative examples of degenerate clusterings resulting from dropping any of the properties of Definition~\ref{def:dense}.
We assume that $k=2$ and $X \subseteq \R^2$ with $d$ being the Euclidean metric.
The examples are depicted in Figure~\ref{fig:violate} below.

\noindent\textbf{Removing connectivity.}
Choose any $\beta,\gamma \le 1$.
Let $X$ consist of disjoint subsets, each one with Euclidean diameter $ \le \epsilon$, and sufficiently far from each other. Label any subsets as $C_1$ and the rest as $C_2$. Note that the second and third property of Definition~\ref{def:dense} are satisfied.

\noindent\textbf{Removing local metric margin.}
Choose any $\gamma \le 1$.
Let $C_1$ be formed by collinear points equally spaced by $\epsilon$, and the same for $C_2$, so that two extremal points of $C_1$ and $C_2$ are at arbitrarily small distance $\delta < \epsilon$.
Note that the first and third property of Definition~\ref{def:dense} are satisfied.

\noindent\textbf{Removing geodesic convexity with margin.}
Choose any $\beta < \frac{1}{2}$.
The set $X$ is formed by collinear points equally spaced by $\frac{\epsilon}{2}$, with the points of cluster $C_1$ interleaved with those of cluster $C_2$.
Note that the first and second property of Definition~\ref{def:dense} are satisfied, but the third is violated for any $\gamma = \omega(\nicefrac{1}{n})$: take the two extreme nodes of $C_1$ and change their shortest path to use a node of $C_2$.

\begin{figure}[h!]
    \centering
    \begin{minipage}[b]{.05\textwidth}
    a)\\[22pt]
    b)\\[8pt]
    c)
    \end{minipage}
    \begin{minipage}[b]{.8\textwidth}
    \includegraphics[scale=.65]{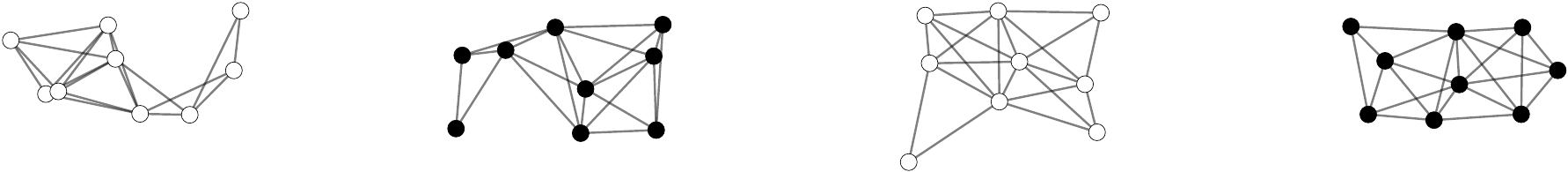}\\[6pt]
    \includegraphics[scale=.65]{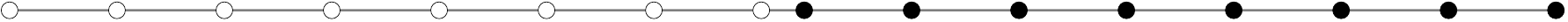}\\[8pt]
    \includegraphics[scale=.65]{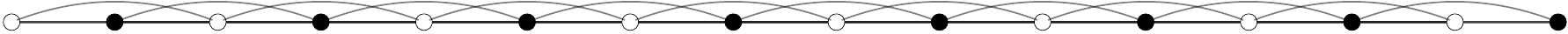}
    \end{minipage}
    \caption{a) a clustering violating connectivity, b) a clustering violating local metric margin, c) a clustering violating geodesic convexity with margin. Empty nodes are in $C_1$, filled nodes are in $C_2$. Shown are the edges of $G_X(\epsilon)$.}
    \label{fig:violate}
\end{figure}

\subsection{Relationship with other clustering notions}
As noted, $(\beta,\gamma)$-convexity is meant to capture density-based clusterings produced by popular algorithms such as Single Linkage and DBSCAN. Those clusterings, however, are in general \emph{not} recoverable with less than $\Omega(n)$ queries, since they allow for $\Omega(n)$ ties (points that can be assigned to one of two clusters in an arbitrary way). Therefore, $(\beta,\gamma)$-convexity should be thought of as an \emph{additional}  property, to be requested \emph{on top of} existing notions of clustering in order to obtain efficient exact recovery. Here, we give two examples of how existing notions of clustering yield $(\beta,\gamma)$-convexity in particular cases.

The first example is DBSCAN, whose parameters are the connectivity radius $r$ and the density parameter $\kappa$. The clustering is defined by looking at $G_X(r)$, clustering together any maximal (sub)tree whose vertices all have degree at least $\kappa$, and assigning each remaining vertex to the same cluster as some of its neighbors (if it has a neighbor). Now, suppose that $G_X(r)$ is formed by $k$ connected components, and each one of them is spanned by a tree on vertices that have degree at least $\kappa$. Then, by taking $\epsilon=r$, one can see that each such component is a cluster that is $(\beta,\gamma)$-dense for $\beta \ge 1$ and every $\gamma > 0$.

The second example is that of spherical clusters with margin~\citep{ashtiani2016clustering}. In this case we use the generalized $(\beta,\gamma)$-convexity, see Definition~\ref{def:dense2}. Let $X \subset \R^{d}$, and suppose that every cluster $C_i$ is contained in some ball $B_i=B(c_i,r_i)$, so that, for some $\delta > 0$, $B(c_i,r_i(1+\delta)) \cap C_j = \emptyset$ for all $j \ne i$. Suppose that each $C_i$ is a set of $(\frac{a}{\delta})^{d}$ points drawn independently and uniformly from $B_i$, for some constant $a > 0$. We claim that the resulting clustering is $(\beta,\gamma)$-convex with high probability, for $\beta = 1$ and any $\gamma > 0$. Indeed, $(\frac{a}{\delta})^{d}$ points draws uniformly at random from $B_i$ are with high probability a $\frac{\delta r_i}{2}$-net for $B_i$ \citep{HDPbook} when $a$ is sufficiently large. In this case, it is easy to see that $C_i$ is connected in $G_X(\epsilon_i)$, where $\epsilon_i=\delta r_i$, thus satifying condition (1) of Definition~\ref{def:dense2}. Moreover, by assumption, for any $x \in C_i$ and any $y \in C_j$ with $j \ne i$ we have $d(y,x) > \epsilon_i$, thus satisfying conditions (2) and (3) of Definition~\ref{def:dense2} for $\beta = 1$ and any $\gamma > 0$.

\paragraph*{Paper organization.}
Section~\ref{sec:same} gives the step-by-step construction of our algorithm for recovering $(\beta,\gamma)$-convex clusterings.
Section~\ref{sec:general} extends $(\beta,\gamma)$-convexity to capture clusters with different radii, and shows how, by introducing \seed\ queries, our algorithm can be extended to this case.
Section~\ref{sec:learn_params} shows how we can learn the radii as well as $\beta$ or $\gamma$, using again \seed\ queries.
Section~\ref{sec:runtime} discusses an efficient implementation of our algorithms.
Section~\ref{sec:lb} presents our query complexity lower bounds.

\section{Exact recovery of $(\beta,\gamma)$-convex clusters}
\label{sec:same}
Throughout this section we assume that $\epsilon,\beta,\gamma$ are known, and that we know a vector of seed nodes $\bs=(s_1,\ldots,s_k)$, so that $s_i \in C_i$ for all $i$. Under these assumptions, our goal is to construct an algorithm, called \RecoverSingleCluster, that for any $i$ returns $C_i$ using $\scO\big(k \log n + k \, \PackNum^*\big(\frac{\beta\gamma}{2+\gamma}\big)\big)$
\scq\ queries and time $\scO(k(n+m))$.
By running \RecoverSingleCluster\ once for each $i$, it is immediate to obtain a full cluster recovery algorithm, \RecoverClustering, with the following guarantees.
Note that $\PackNum^*\big(\frac{\beta\gamma}{2+\gamma}\big) \le (\nicefrac{6}{\beta\gamma})^{\dens(X)}$ since $\PackNum^*(\eta) \le (\nicefrac{2}{\eta})^{\dens(X)}$ and $\gamma \le 1$.
\begin{theorem}
\label{th:same_radii}
Suppose $\clus$ is $(\beta,\gamma)$-convex (Definition~\ref{def:dense}).
Then \RecoverClustering$(X,\epsilon,\bs)$ deterministically returns $\clus$ in time $\scO(k^2(n+m))$ using $\scO\big(k^2 \log n + k^2 \, (\nicefrac{6}{\beta\gamma})^{\dens(X)}\big)$ \scq\ queries.
\end{theorem}
The rest of this section describes \RecoverSingleCluster\ and proves Theorem~\ref{th:same_radii}, with the exception of the running time bound, which is proven in Section~\ref{sec:runtime}. Unless otherwise specified, from now on $G$ stands for $G_X(\epsilon)$. 

\subsection{Margin-based separation, and binary search on shortest paths.}
We start with a subroutine \BallSeparator\ (Margin-Based Separator) that, given an input set $Z$ and the cluster index $i$, computes $Z \cap C_i$.
The routine uses the local metric margin and is efficient when the metric radius of $Z$ is small.
\begin{algorithm2e}[h!]
\caption{\BallSeparator($Z,\epsilon,u$)}
\label{algo:ballsep}
$U := \emptyset$\;
\textbf{for }{each connected component $H$ of $G_{Z}(\beta\epsilon)$} \textbf{do}\;{
 \Indp choose any $x \in V(H)$\;
 \lIf{$\scq(x,u)= 1$}{add $V(H)$ to $U$ \label{line:x_scq}}
 \Indm
}
\Return $U$\;
\end{algorithm2e}
\begin{lemma}
\label{lem:ballsep}
For any $u \in X$ and any $Z \subseteq X$ such that $Z \subseteq B(z,r)$ for some $z \in X$ and $r < \infty$, \BallSeparator$(Z,\epsilon,u)$ returns $Z\, \cap\, C_i$ using $\PackNum^*\big(\frac{\beta\epsilon}{r}\big)$ \scq\ queries, where $i=\clus(u)$.
\end{lemma}
\begin{proof}
Let $G_Z=G_{Z}(\beta\epsilon)$.
If $H$ is any connected component of $G_Z$, then $d(x,y) \le \beta \epsilon$ for all $x,y \in H$.
Thus, a repeated application of the local metric margin implies that all nodes of $H$ belong to the same cluster.
Therefore, either $V(H) \cap C_i = \emptyset$, or $V(H) \subseteq C_i$.
This shows that $x$ is added to $U$ if and only if $x \in C_i$, proving the correctness.
For the query complexity, let $P$ be the set of points $x$ queried by the algorithm at line~\ref{line:x_scq}.
Clearly $P$ is an independent set in $G_Z$ and, thus, $d(x,y) > \beta\epsilon$ for all distinct $x,y \in P$.
By the definition of $\PackNum^*$ this implies that $|P| \le \PackNum^*(\frac{\beta\epsilon}{r})$.
\end{proof}

Next, we introduce a condition for finding efficiently a cut edge of $C_i$.
A path $\pi=(x_1,\ldots,x_{|\pi|})$ is \textsl{$C_i$-prefixed} if there exists an index $j^* \in \big[|\pi|\big]$ such that $x_{j} \in C_i$ if and only if $j \in \{1,\ldots,j^*\}$.
\begin{lemma}
\label{lem:monoRi}
Let $C_i \subseteq R \subseteq X$ such that $G[R]$ is connected.
Then, in $G[R]$, any shortest path between any $s_i \in C_i$ and any $s \in R \setminus C_i$ is $C_i$-prefixed.
\end{lemma}
\begin{proof}
Suppose by contradiction that in $G[R]$ there exists a shortest path $\pi$ between $s_i \in C_i$ and $s \in R \setminus C_i$ that is not $C_i$-prefixed. Then some prefix $\pi'$ of $\pi$ is a shortest path between $s_i \in C_i$ and $s_i' \in C_i$ that intersects $R \setminus C_i$. Now observe that $\pi'$ is a shortest path in $G$, too. This holds because any shortest path between $s_i$ and $s_i'$ in $G$ lies inside $G[C_i]$ by geodesic convexity, and thus inside $G[R]$ as $C_i \subseteq R$. By geodesic convexity this implies that $\pi' \subseteq G[C_i]$, a contradiction.
\end{proof}

Finally, we observe that, in a $C_i$-prefixed simple path, a cut edge of $C_i$ can be found via binary search from the endpoints of the path. This yields a subroutine \FindCutEdge\ with the following guarantees (pseudocode in Appendix~\ref{app:same}):
\begin{observation}
\label{obs:bin_search}
Given a simple $C_i$-prefixed path $\pi$, \FindCutEdge$(\pi)$ returns the unique cut edge of $C_i$ in $\pi$ using $\scO(\log n)$ \scq\ queries.
\end{observation}

\subsection{Cluster separators}
\label{sub:sep}
We introduce the notion of cluster separator, which is at the heart of our algorithm.
\begin{definition}
A partition $(S_i, S_j)$ of $X$ is an $(i,j)$-separator of $X$ if $S_i \cap C_j = S_j \cap C_i = \emptyset$.
\end{definition}
This is similar to half-spaces in abstract closure systems~\citep{seiffarth2019maximal}, where we would have $C_i \subseteq S_i$ and $C_j \subseteq S_j$.
We use the weaker condition $S_i \cap C_j = S_j \cap C_i = \emptyset$ because in some of our algorithms $X$ will be a subset of the input set, in which case $C_j \subseteq X$ might not hold. On the other hand, we will always make sure that $C_i \subseteq X$ holds.

Now, if $C_i \subseteq X$ and $(S_i,S_j)$ is an $(i,j)$-separator for $X$, then $C_i \subseteq S_i$ but $C_j \cap S_i = \emptyset$.
Thus, if we could compute an $(i,j)$-separator for all $j \ne i$, then we could compute $C_i$ by a simple set intersection.
Unfortunately, it is not clear how to compute $(S_i,S_j)$ for an arbitrary $j$, even given the seed node $s_j$.
However, as we shall see, we can compute $(S_i,S_j)$ if we know a cut edge $(u_i,u_j)$ between $C_i$ and $C_j$.
The trick is to take each $x \in X$ and look at its distance $d_G(x,u_i)$ from $u_i$.
If $d_G(x,u_i) < \frac{1}{\gamma}$, then we learn whether $u_i \in C_i$ using \BallSeparator.
If instead $d_G(x,u_i) \ge \frac{1}{\gamma}$, then the comparison $d_G(x,u_i) \le d_G(x,u_j)$ tells us whether $x$ should be in $S_i$ or in $S_j$, without using any query.
This is implied by geodesic convexity through the following lemma, proven in Appendix~\ref{app:same}:
\begin{restatable}{relemma}{ziplemma}
\label{lem:zip}
Let $(u_i,u_j) \in G$ be a cut edge between $C_i$ and $C_j$.
For all $x \in X$ with $\frac{1}{\gamma} \le d_G(u_i,x) < \infty$, if $d_G(u_i,x) \le d_G(u_j,x)$ then $x \notin C_j$, and if $d_G(u_i,x) \ge d_G(u_j,x)$ then $x \notin C_i$.
\end{restatable}
\noindent The intuition is that $d_G(u_i,x) \le d_G(u_j,x)$ and $x \in C_j$ cannot both hold, because this would violate the geodesic convexity of $C_j$, and the same holds when $i$ and $j$ are exchanged.

The above discussion leads to \Cut\ (Algorithm~\ref{algo:cut_2}), whose correctness and query cost are proven in Lemma~\ref{lem:algosep}.
Clearly enough, to use \Cut\ we need to compute the cut edge $(u_i,u_j)$, and we show how to do so in the next sections.

\begin{algorithm2e}[h!]
\caption{\Cut($G,u_i,u_j$)}
\label{algo:cut_2}
$Z := \{x \in V(G) \,:\, d_G(u_i,x) < \nicefrac{1}{\gamma}\}$\;
$Z_i := $ \BallSeparator($Z,\epsilon,u_i$),\quad $Z_j := Z \setminus Z_i$\;
$U_i := \emptyset$,\quad  $U_j := \emptyset$\;
\textbf{for }{each $x \in V(G) \setminus Z$ \label{line:cycle}} \textbf{do}\;{ 
  \Indp
  \textbf{if}\, {$d_{G}(x,u_i) \le d_{G}(x,u_j)$}\, \textbf{then}\, {add $x$ to $U_i$ \label{line:cut_ui}}\, \textbf{else}\, {add $x$ to $U_j$ \label{line:cut_uj}}\;
  \Indm
}
\Return $(Z_i \cup U_i,\, Z_j \cup U_j)$\;
\end{algorithm2e}
\begin{lemma}
\label{lem:algosep}
Suppose $(u_i,u_j) \in G$ is a cut edge between $C_i$ and $C_j$.
Then, \Cut$(G,u_i,u_j)$ returns a pair $(S_i,S_j)$ that is an $(i,j)$-separator of $V(G)$, using $\scO(\PackNum^*(\beta\gamma))$ \scq\ queries.
\end{lemma}
\begin{proof}
The query bound follows from Lemma~\ref{lem:ballsep} by observing that $Z \subseteq B(u_i, \nicefrac{\epsilon}{\gamma})$. For the correctness, we show that $(Z_i,Z_j)$ is an $(i,j)$-separator of $Z$ and $(U_i,U_j)$ is an $(i,j)$-separator of $V(G) \setminus Z$. For $Z$, Lemma~\ref{lem:ballsep} guarantees that $Z_i = C_i \cap Z$, and the algorithm sets $Z_j = Z \setminus Z_i$. So $Z_i \cap C_j = Z_j \cap C_i = \emptyset$, and $(Z_i,Z_j)$ is an $(i,j)$-separator of $Z$. Consider now any $x \in V(G) \setminus Z$. By definition of $Z$, we have $d_G(x,u_i) \ge \frac{1}{\gamma}$. Therefore, by Lemma~\ref{lem:zip}, if the algorithm assigns $x$ to $U_i$ then $x \notin C_j$, and if the algorithm assigns $x$ to $U_j$ then $x \notin C_i$. Therefore $U_i \cap C_j = U_j \cap C_i = \emptyset$, and $(U_i,U_j)$ is an $(i,j)$-separator of $V(G) \setminus Z$.
\end{proof}

\subsection{Recovering a single cluster}
\label{sub:single}
We can now describe \RecoverSingleCluster\ (Algorithm~\ref{algo:recover_single}).
The algorithm starts by computing $R_i$, the set of nodes reachable from $s_i$ in $G$, and the corresponding induced subgraph $G_i=G[R_i]$.
Note that, by the connectedness of the clusters, initially $R_i$ is simply the union of $C_i$ and zero or more other clusters.
Now, we search for some seed node $s_h \in \bs$, such that $s_h \in R_i$ but $s_h \ne s_i$.
If no such node is found, then $R_i = C_i$ and we are done.
Otherwise, we compute the shortest path $\pi$ between $s_h$ and $s_i$ in $G_i$.
By Lemma~\ref{lem:monoRi}, $\pi$ is $C_i$-prefixed, and so by Observation~\ref{obs:bin_search} we can find a cut edge $(u_i,u_j)$ of $C_i$ with $\scO(\log n)$ \scq\ queries.
With the cut edge $(u_i,u_j)$, we can compute an $(i,j)$-separator $(S_i,S_j)$ of $X=V(G)$ using \Cut.
Finally, we update $G_i$ to be the connected component of $s_i$ in $G[R_i \cap S_i]$, and $R_i$ to be its node set.
By definition of $(S_i,S_j)$, this update removes from $G_i$ all points of $C_j$, so we have reduced by at least one the number of clusters other than $C_i$ intersected by $R_i$.
After at most $k-1$ of these rounds, we will be left with $R_i=C_i$.

Unfortunately, this process has a problem: we can run out of seeds in $R_i$.
Indeed, by taking $R_i \cap S_i$, we could remove every seed node $s_h$, even if $R_i \cap S_i$ still contains points of $C_h$.
If this is the case, then, even though $R_i \ne C_i$, \RecoverSingleCluster\ would be stuck, unable to compute a new cut edge to remove some cluster from $R_i$.
One is tempted to ignore the fact that $s_h \notin R_i$, and simply compute the shortest path between $s_h$ and $s_i$ for all $h \ne i$, obtaining a set of different cut edges.
This however does not work, as all those shortest paths could use the same cut edge $(u_i,u_j)$.

We bypass this obstacle by exploiting the separators found by \RecoverSingleCluster.
By carefully analysing the cuts induced by those separators, we devise an algorithm, \FindNewSeed, that either finds some new seed $s_h \in R_i \setminus C_i$ or certifies that $R_i = C_i$.
\FindNewSeed\ is invoked by \RecoverSingleCluster\ at the beginning of each round, and we describe it in Section~\ref{sub:findseed}.

\begin{algorithm2e}[ht!]
\caption{\RecoverSingleCluster($G,\epsilon,\bs,i$)}
\label{algo:recover_single}
$G_i := \{$the connected component of $s_i$ in $G\}$,\,\,  $R_i := V(G_i)$\;
$\bu := $ an empty vector\;
\textbf{while} true \textbf{do}\; {
\Indp
  $s_h := \FindNewSeed(G,R_i,\epsilon,\bs,\bu,i)$\; \label{line:findnewseed}
  \lIf{$s_h = \nil$}{stop and \Return $R_i$ \label{line:return}}
    $\pi(s_i,s_h) := $ ShortestPath$(G[R_i],s_i,s_h)$\;
    $(u_i,u_j) :=$ \FindCutEdge$(\pi(s_i,s_h))$\; \label{line:ice} 
    add $u_i$ to $\bu$\;
    $(S_i,S_j) :=$ \Cut$(G,u_i,u_j)$\; \label{line:clustersep}
    $G_i := \{$the connected component of $s_i$ in $G[R_i \cap S_i]\}$,\,\, $R_i := V(G_i)$\; \label{line:connect-1}
\Indm
}
\end{algorithm2e}

\begin{restatable}{relemma}{singlereclemma}
\label{lem:single_recovery}
\RecoverSingleCluster$(G,\epsilon,\bs,i)$ returns $C_i$ using $\scO\big(k \log n + k \, \PackNum^*(\frac{\beta\gamma}{2+\gamma})\big)$ \scq\ queries.
\end{restatable}
\noindent The proof is along the lines of the discussion above, see Appendix~\ref{app:same}.

\subsection{Finding new seed nodes}
\label{sub:findseed}
We describe \FindNewSeed, which finds a node $s_h \in R_i \setminus C_i$ if $R_i \setminus C_i \ne \emptyset$, and otherwise detects that $R_i = C_i$ and returns $\nil$. The key idea behind \FindNewSeed\ is the following: if $R_i$ does not contain any seed $s_h \in \bs$ other than $s_i$, then for each $h \ne i$ either $C_h \cap R_i = \emptyset$, or, by the connectedness of $G[C_h]$, the cut $\edgecut(R_i)$ must contain some edge of $G[C_h]$. Therefore, the task boils down to finding such an edge, or deciding that no such edge exists. Clearly, we cannot just check all edges in $\edgecut(R_i)$, as this would require too many queries. Thus, we proceed as follows.

Consider the beginning of a generic iteration of \RecoverSingleCluster, and let $\bu$ be the set of all nodes $u_i$ that appeared in a cut edge used in some previous iteration. First, for every $u \in \bu$, we consider the set $Z$ of nodes $x \in R_i$ such that $d_G(u,x) < \frac{2}{\gamma}+1$. Then, like we did in \Cut, we use \BallSeparator\ to recover efficiently the subset $Z \setminus C_i$. If this subset is nonempty, then we just return any $x \in Z \setminus C_i$ and we are done. If after considering every $u \in \bu$ we have not found a seed, then we turn to the remaining nodes, that is, all nodes $x$ such that $d_G(u,x) \ge \frac{2}{\gamma}+1$ for all $u \in \bu$. In this case, as a consequence of geodesic convexity with margin we prove the following structural result: if $x$ has an edge $(x,y) \in \edgecut(R_i)$, then $x \notin C_i$. So, if any such $x$ exists, which we can check without making any query, then we can again return $x$. If both attempts to find $x \in R_i \setminus C_i$ fail, we can show that necessarily $R_i=C_i$.

Lemma~\ref{lem:findseed} below states the guarantees of \FindNewSeed.
Its proof is found in Appendix~\ref{app:same}.
\begin{algorithm2e}[h!]
\caption{\FindNewSeed($G,R_i,\epsilon,\bs,\bu,i$)}
\label{algo:find_seed}
\lIf{$\bs \cap R_i \ne \{s_i\}$ \label{line:if1}}{
  \Return any $s \in \bs \cap R_i \setminus \{s_i\}$ \label{line:ret_sh}
}
\textbf{for} {each $u \in \bu$} \textbf{do} {\label{line:for}}\; {
\Indp
   	$Z = \{x \in R_i \,:\, d_G(u,x) < \nicefrac{2}{\gamma}+1\}$\;
    $Z_i :=$ \BallSeparator($Z,\epsilon,u$)\; \label{line:Zi}
    \lIf{$Z \setminus Z_i \ne \emptyset$ \label{line:Z_Zi}}{
      \Return any $x \in Z \setminus Z_i$ \label{line:ret_ballsep}
    }
\Indm
}
\lIf{$\exists \, (x,y) \in \edgecut(R_i)$ such that $\forall u \in \bu \,:\, d_G(u,x) \ge \nicefrac{2}{\gamma}+1$ \label{line:if2}}{
  \Return any such $x$ \label{line:ret_cross}
}\lElse{
 \Return \nil
}
\end{algorithm2e}
\begin{restatable}{relemma}{findseedlemma}
\label{lem:findseed}
Consider the beginning of any iteration of \RecoverSingleCluster$(G,\epsilon,\bs,i)$.
Then, the call to \FindNewSeed$(G,R_i,\epsilon,\bs,\bu)$ returns a point $x \in R_i \setminus C_i$ if $R_i \ne C_i$, or \nil\ if $R_i = C_i$, using at most $k \, \PackNum^*(\frac{\beta\gamma}{2+\gamma})$ \scq\ queries.
Moreover, \FindNewSeed\ can be adapted so as to make at most $k \, \PackNum^*(\frac{\beta\gamma}{2+\gamma})$ \scq\ queries over the entire execution of \RecoverSingleCluster.
\end{restatable}

\section{Extension to clusters with different radii}
\label{sec:general}
In this section we generalize the notion of $(\beta,\gamma)$-convexity (Definition~\ref{def:dense}) so to allow each cluster $C_i$ to have its own radius, denoted by $\rad_i$.
Then, by using \seed\ queries, we will extend our cluster recovery algorithm to this generalized setting.
For technical reasons, we need to strengthen geodesic convexity in a hereditary fashion.
\begin{definition}[generalized $(\beta,\gamma)$-convex clustering]
\label{def:dense2}
For any $\beta,\gamma \in (0,1]$, a $k$-clustering $\clus=(C_1,\ldots,C_k)$ of $X$ is $(\beta,\gamma)$-convex if $\,\forall\, i \in [k] : \exists\,\rad_i > 0$ such that the following properties hold:
\begin{enumerate}[leftmargin=17pt,topsep=1pt,parsep=2pt,itemsep=0pt]
    \item \emph{connectedness}: the subgraph induced by $C_i$ in $G_X(\rad_i)$ is connected 
    \item \emph{local metric margin}: for all $x,y \in X$, if $x \in C_i$ and $y \notin C_i$, then $d(x,y) > \beta\,\! \epsilon_i$
    \item \emph{geodesic convexity with margin}: for any $\epsilon \le \rad_i$, if $x,y \in C_i$ and $d_{G_X(\epsilon)}(x,y) < \infty$, then in $G_X(\epsilon)$ any simple path between $x$ and $y$ of length at most $(1+\gamma) d_{G_X(\epsilon)}(x,y)$ lies entirely in $C_i$ 
\end{enumerate}
\end{definition}
In Lemma~\ref{lem:mineps} in the Appendix, we show that if the conditions above are satisfied, then they are satisfied in particular by the smallest $\rad_i$ such that $C_i$ is connected in $G_X(\rad_i)$.\footnote{Note that this is different from requiring that $G_X(\rad_i)[C_i]$ is connected; here we are only requiring that any two points of $C_i$ have a connecting path in $G_X(\rad_i)$. Lemma~\ref{lem:mineps} however shows that the smallest $\rad_i$ that satisfies one condition is also the smallest $\rad_i$ that satisfies the other condition.}
Therefore, without loss of generality, we assume that each $\rad_i$ satisfies this minimality assumption.

We turn to the recovery of $\clus$.
We show that, with some care, the problem can be reduced to the case of identical radii, at the price of making $\scO(k^2)$ \seed\ queries.
This gives an algorithm \ReduceRecover\ with the following guarantees, where $\beps=(\epsilon_1,\ldots,\epsilon_k)$ is the vector of the radii:
\begin{restatable}{retheorem}{diffradtheorem}
\label{th:diff_rad}
Suppose $\clus$ is $(\beta,\gamma)$-convex (Definition~\ref{def:dense2}).
Then, \ReduceRecover$(X,\beps,\bs)$ deterministically returns $\clus$, has the same runtime as \RecoverClustering$(X,\beps,\bs)$, and uses the same number of \scq\ queries as \RecoverClustering$(X,\beps,\bs)$, plus at most $\scO(k^2)$ \seed\ queries.
\end{restatable}
The basic idea of \ReduceRecover\ is to invoke \RecoverSingleCluster\ for each $i \in [k]$, as we did for the case of identical radii, using the graph $G_X(\epsilon_i)$ when we want to recover $C_i$. This does not work straight away, however: in $G_X(\epsilon_i)$, any cluster $C_j$ with $\epsilon_j < \epsilon_i$ is by definition not required to satisfy geodesic convexity, which means that \RecoverSingleCluster\ can fail. However, we can show that this approach works if we adopt the following precautions:
\begin{enumerate}\itemsep2pt
    \item recover the clusters in nondecreasing order of radius
    \item when recovering $C_i$, restrict $G_X(\rad_i)$ to its connected component containing $s_i$
    \item after recovering $C_i$, delete it from $X$.
\end{enumerate}
Note that, for each $i$, this procedure works on a potentially different graph --- thresholded at a different radius, and containing only a subset of the original points. Thus, its correctness may not be obvious. In particular, it may not be obvious that the clustering induced by the sub-instance used at the $i$-th iteration is $(\beta,\gamma)$-convex, which is necessary for \RecoverSingleCluster\ to work. We show that it is: for every $i$, at the $i$-th iteration, the residual clustering is $(\beta,\gamma)$-convex. Thus the input to \RecoverSingleCluster\ satisfies the hypotheses of Lemma~\ref{lem:single_recovery}, and by that lemma, \RecoverSingleCluster\ will return $C_i$ using $\scO\big(k \log n + k \, \PackNum^*(\frac{\beta\gamma}{2+\gamma})\big)$ \scq\ queries, as desired. In all this, the role of \seed\ queries is to find one seed node $s_h$ for each cluster in the connected component of $G_X(\rad_i)$ containing $s_i$, as required by \RecoverSingleCluster.

The formal construction of \ReduceRecover\ and the proof of Theorem~\ref{th:diff_rad} are given in Appendix~\ref{app:general}.

\section{Learning the radii and the convexity parameters}
\label{sec:learn_params}
In this section we show how to use \seed\ queries to learn the cluster radii and to deal with the case where one of $\beta$ or $\gamma$ is unknown.
For learning the radii, we prove:
\begin{restatable}{retheorem}{epsitheorem}
\label{thm:epsilons}
Suppose $\clus$ is $(\beta,\gamma)$-convex (Definition~\ref{def:dense2}).
Then, the cluster radii $\rad_1,\ldots,\rad_k$ can be learned using $\scO(k \log n)$ \seed\ queries in time $\scO(m \, \alpha(m,n) + k n \log n)$, where $\alpha(m,n)$ is the functional inverse of the Ackermann function.\footnote{For all practical purposes, $\alpha$ can be considered constant. For instance, $\alpha(m,m) \le 4$ for all $m \le \frac{1}{8}2^{2^{2^{2^{65536}}}}$.}
\end{restatable}
This result hinges on three observations. First, as said above, $\epsilon_i$ is actually the smallest $\epsilon$ such that all nodes of $C_i$ belong to a single connected component of $G_X(\epsilon)$. Second, with $\scO(1)$ \seed\ queries, we can test whether $C_i$ belongs to a single connected component of $G$, for any given graph $G$, see Claim~\ref{claim:isconnected}. Third, if $T$ is any minimum spanning tree of $\scG$, then the connected components of $G_X(\epsilon)$ are exactly the connected components of $T(\epsilon)$, see Claim~\ref{claim:mst}. Our algorithm starts by computing $T$, which takes time $\scO(m\, \alpha(m,n))$ where $\alpha$ is inverse Ackermann, see~\citep{ChazelleMST}. Then, for each cluster $C_i$, we perform a binary search to find the smallest edge weight $\epsilon$ such that $C_i$ is connected in $T(\epsilon)$. All details are given in Appendix~\ref{app:learn_params}. In Section~\ref{sec:lb}, we also prove a nearly-matching lower bound of $\Omega(k \log\frac{n}{k})$ queries.

We conclude by discussing the case where one among $\beta$ and $\gamma$ is unknown.
Equipped with the \seed\ queries, we make a series of halving guesses for $\beta$ or $\gamma$, until we detect that the clustering is correct.
This yields the following result (proof in Appendix~\ref{app:learn_params}):
\begin{restatable}{retheorem}{guessingtheorem}
\label{lem:guessing}
Suppose $\clus$ is $(\beta,\gamma)$-convex (Definition~\ref{def:dense2}), and let $R = \log(\frac{4}{\beta\gamma})\dens(X)$.
If only one between $\beta$ and $\gamma$ is unknown, then we can recover $\clus$ with a multiplicative overhead of $\scO(R)$ in both query cost and running time, plus $\scO(k^2 R)$ \scq\ queries and $\scO(k R)$ \seed\ queries.
This applies to each one of our algorithms (i.e., with radii that are identical or not, known or unknown).
\end{restatable}

\section{Bounds on the running time}
\label{sec:runtime}
All our algorithms admit efficient implementations, with a running time linear in the size of $\scG$ (or essentially linear, see Section~\ref{sec:learn_params}). Here, we give a quick overview of these implementations; for a more complete discussion, see Appendix~\ref{app:runtime}.

Recall that our input is the weighted graph $\scG=(X,\scE,d)$, where $(u,v) \in \scE$ if and only if $d(u,v) < \infty$. Recall also that $n=|X|$ and $m=|\scE|$.
We assume that $d$ can be accessed in constant time, which can be achieved with a hash map, whose construction takes time $O(m)$~\citep{perfecthash}.
We further assume that, for any graph $G$ and for any $x \in V(G)$, the edges incident to $x$ can be listed in constant time per edge.
Under these assumptions, the following basic fact holds:
\begin{observation}
\label{obs:G_x_eps_time}
Given $\scG$, for any $\epsilon$, the graph $G_X(\epsilon)$ can be computed in time $\scO(n+m)$. This holds in general for thresholding any subgraph of $\scG$.
\end{observation}

Using Observation~\ref{obs:G_x_eps_time}, we can easily implement \RecoverSingleCluster\ so that it runs in time $\scO(k^2(n+m))$: essentially, for $\scO(k^2)$ times the algorithm computes the distances of all nodes of $G_X(\epsilon)$ from some given node $u$, which takes time $\scO(n+m)$ via breadth-first search. Since \RecoverSingleCluster\ is invoked once per each cluster, this would give a total running time of $\scO(k^3(n+m))$ for both \RecoverClustering\ and \ReduceRecover. With some extra effort, however, we show how to adapt \RecoverSingleCluster\ so that it runs $\scO(k(n+m))$, by amortizing in particular the cost of its subroutine \FindNewSeed. In the end, we obtain a total running time of $\scO(k^2(n+m))$ for both our cluster recovery algorithms, \RecoverClustering\ and \ReduceRecover.

\section{Lower Bounds}
\label{sec:lb}
In this section we show that some of our parameters and assumptions are necessary, and in particular: (1) in general, to recover a $(\beta,\gamma)$-convex clustering, any algorithm needs $\Omega(2^{\dens(X)})$ queries; (2) to recover a $(\beta,\gamma)$-convex clustering without initial seed nodes, any algorithm needs $\Omega(n)$ \scq\ queries; (3) to learn the radii of $k$ clusters, any algorithm needs $\scO(k \log \frac{n}{k})$ \scq\ and/or \seed\ queries.
The full proofs are deferred to Appendix~\ref{app:lb}.

\begin{restatable}[Dependence on $\dens(X)$.]{retheorem}{packdeptheorem}
\label{th:pack_dep}
Choose any $\beta,\gamma \in (0,1)$.
There is a distribution of $(\beta,\gamma)$-convex $2$-clusterings $\clus$ (Definition~\ref{def:dense} or Definition~\ref{def:dense2}), where $n=|X|=2^{\dens(X)}$ is arbitrarily large, such that any algorithm (even randomized) needs $\Omega(2^{\dens(X)})$ \scq\ and/or \seed\ queries to recover $\clus$ with constant probability.
This holds even if $\beta,\gamma,\epsilon$ are known.
\end{restatable}

\begin{restatable}[Necessity of seeds.]{retheorem}{necseedstheorem}
\label{th:nec_seeds}
Choose any $\beta,\gamma \in (0,1]$.
There is a distribution of $(\beta,\gamma)$-convex $2$-clusterings $\clus$ (Definition~\ref{def:dense} or Definition~\ref{def:dense2}), where $X \subseteq \R^2$ with $n=|X|$ arbitrarily large and $d$ the Euclidean distance, such that any algorithm (even randomized) needs $\Omega(n)$ \scq\ queries to recover $\clus$ with constant probability if no seed nodes are given.
This holds even if $\gamma, \alpha, \epsilon$ are known.
\end{restatable}

\begin{restatable}[Cost of learning the radii.]{retheorem}{radiilowertheorem}
\label{thm:learn_radii_lb}
For any $k \ge 2$, and any sufficiently large $n$, there exists a distribution of $(\nicefrac{1}{2},1)$-convex $k$-clusterings (Definition~\ref{def:dense} or Definition~\ref{def:dense2}) on $n$ points such that any algorithm (even randomized) needs $\Omega(k \log \frac{n}{k})$ \seed\ and/or \scq\ queries to learn the radii of all clusters with constant probability.
\end{restatable}

\acks{The authors gratefully acknowledge partial support by the Google Focused Award ``Algorithms and Learning for AI'' (ALL4AI).  Nicolò Cesa-Bianchi is also supported by the MIUR PRIN grant Algorithms, Games, and Digital Markets (ALGADIMAR) and by the EU Horizon 2020 ICT-48 research and innovation action under grant agreement 951847, project ELISE (European Learning and Intelligent Systems Excellence).}

\bibliography{biblio}

\clearpage
\appendix

\section{Supplementary material for Section~\ref{sec:prelim}}
\label{app:prelim}
\begin{lemma}
\label{lem:doubling}
If $X$ is a metric space, then $\doubl(X) \le \dens(X) \le 2 \doubl(X)$ where $\doubl(X)$ and $\dens(X)$ are respectively the doubling dimension and the density dimension of $X$.
\end{lemma}
\begin{proof}
Recall that $\doubl(X) = \log_2 D(X)$, where $D(X)$ is the doubling constant of $X$:
\begin{align}
    D(X) = \min\left\{ D \in \Nat \,:\, (x \in X) \wedge (r > 0) \Rightarrow \CoverNum\!\left(B(x,r), \frac{r}{2}\right) \le D \right\}
\end{align}
where $\CoverNum(K,\eta)$ is the covering number of $K$, that is, the smallest number of closed balls of radius $\eta$ whose union contains $K$.
We recall from Section~\ref{sec:prelim} that $\dens(X)= \log_2 \mu(X)$, where: 
\begin{align}
    \mu(X) = \min\left\{ \mu \in \Nat \,:\, (x \in X) \wedge (r > 0) \Rightarrow \PackNum\!\left(B(x,r), \frac{r}{2}\right) \le \mu \right\} 
\end{align}
Now, since we are in a metric space, we have the well-known relationship:
\begin{align}
    \PackNum(K,2\eta) \le \CoverNum(K,\eta) \le \PackNum(K,\eta)
\end{align}
On the one hand, $\CoverNum(K,\eta) \le \PackNum(K,\eta)$ implies $D(X) \le \mu(X)$, and therefore $\doubl(X) \le \dens(X)$.
On the other hand, $\PackNum(K,2\eta) \le \CoverNum(K,\eta)$ implies:
\begin{align}
    \PackNum(B(x,r),\nicefrac{r}{2}) &\le \CoverNum(B(x,r),\nicefrac{r}{4})
    \\ &\le \CoverNum(B(x,r),\nicefrac{r}{2}) \cdot D(X)
    \\ &\le D(X) \cdot D(X)
\end{align}
Therefore, $\mu(X) \le D(X)^2$, and $\dens(X) \le 2\doubl(X)$.
\end{proof}

\section{Supplementary material for Section~\ref{sec:same}}
\label{app:same}

\subsection{Pseudocode of \FindCutEdge}
\begin{algorithm2e}[h!]
\label{algo:ice}
\caption{\FindCutEdge($\pi(s_i,s_h)$)}
\lIf{$|\pi(s_i,s_h)| = 1$}{
  \Return $\pi(s_i,s_h)$
}
  choose a median node $x \in \pi(s_i,s_h)$\;
  \lIf{\scq$(s_i,x) = +1$}{\Return~\FindCutEdge($\pi(x,s_h)$)}
  \lElse{\Return~\FindCutEdge($\pi(s_i,x)$)}
\end{algorithm2e}

\subsection{Proof of Lemma~\ref{lem:zip}}
\ziplemma*
\begin{proof}
Suppose $d_G(u_i,x) \le d_G(u_j,x)$ and $x \in C_j$; we show this leads to a contradiction.
Consider the path $\pi=\pi(x,u_i)+(u_i,u_j)$.
First, $\pi$ is a simple path.
If this was not the case, then we would have $u_j \in \pi(x,u_i)$; but this would imply $d_G(u_j,x) \le d_G(u_i,x)-1$, contradicting our assumption $d_G(u_i,x) \le d_G(u_j,x)$.
Second, we have:
\begin{align}
|\pi| &= d_{G}(x,u_i)+1
\\ &\le d_{G}(x,u_j)+1 && \text{since } d_G(u_i,x) \le d_G(u_j,x)
\\ &\le d_{G}(x,u_j) \, (1+\gamma) && \text{since }d_{G}(x,u_j) \ge d_{G}(x,u_i) \ge \frac{1}{\gamma}
\end{align}
Thus, $\pi$ is a simple path between two nodes of $C_j$, with length at most $(1+\gamma)$ times their distance in $G$, and containing a node of $C_i$.
This violates the geodesic convexity of $C_j$.
We conclude that $x \notin C_j$.
The other case is symmetric.
\end{proof}

\subsection{Proof of Lemma~\ref{lem:single_recovery}}
\singlereclemma*
\begin{proof}
First, we prove the correctness.
For each $\ell=1,2,\ldots$, we denote by $R_i^{\ell}$ the set $R_i$ at the beginning of the $\ell$-th iteration, and by $\kappa(R_i^{\ell})$ the number of distinct clusters that $R_i^{\ell}$ intersects.
We show that, at the beginning of the $\ell$-th iteration, the following invariants hold:
\begin{enumerate}[topsep=2pt,parsep=0pt,itemsep=2pt]
\item[I1.] $G[R_i^{\ell}]$ is connected
\item[I2.] $C_i \subseteq R_i^{\ell}$
\item[I3.] $\kappa(R_i^{\ell}) \le \kappa(R_i^1) - (\ell-1)$
\end{enumerate}
Note that I2 and I3 imply that the algorithm terminates by returning $C_i$ (line \ref{line:return}) after at most $k$ iterations.
Invariant I1 holds by the construction of $R_i^{\ell}$, so we focus on proving I2 and I3.

Suppose first $\ell=1$.
Then, I2 holds by the assumptions of the lemma, and I3 is trivial.
Suppose then I1, I2, I3 hold for some $\ell \ge 1$, and that iteration $\ell+1$ exists; we show that I2, I3 hold at iteration $\ell+1$ as well.
First, if iteration $\ell+1$ exists, then $C_i \subsetneq R_i^{\ell}$.
In this case, by Lemma~\ref{lem:findseed}, \FindNewSeed\ will return a node $s_h \in R_i^{\ell} \setminus C_i$.
Since $G[R_i^{\ell}]$ is connected by I1, the shortest path $\pi(s_i,s_h)$ exists in $G[R_i^{\ell}]$.
Because of Lemma~\ref{lem:monoRi} applied to $R=R_i^{\ell}$, such a path is also $C_i$-prefixed.
Therefore, \FindCutEdge$(\pi(s_i,s_h))$ returns a cut edge $(u_i,u_j)$ of $C_i$ in $\pi(s_i,s_h)$.
At this point the hypotheses of Lemma~\ref{lem:algosep} are satisfied, and therefore the output $(S_i,S_j)$ of \Cut$(G,u_i,u_j)$ is an $(i,j)$-separator of $V(G)$.
Hence, $C_i \subseteq S_i$ and $C_j \cap S_i = \emptyset$.
Therefore, $C_i \subseteq R_i^{\ell} \cap S_i$.
This implies that the connected component of $s_i$ in $G[R_i^{\ell} \cap S_i]$ still contains $C_i$.
The vertex set of this connected component is precisely $R_i^{\ell+1}$ (line~\ref{line:connect-1}).
Therefore, I2 holds at iteration $\ell+1$.
Moreover, observe that $u_j \in R_i^{\ell}$, since $u_j \in \pi(s_i,s_h) \subseteq R_i^{\ell}$.
Therefore $R_i^{\ell} \cap C_j \ne \emptyset$.
However, by construction, $R_i^{\ell+1} \subseteq  R_i^{\ell} \cap S_i$ and $S_i \cap C_j = \emptyset$ since $(S_i,S_j)$ is an $(i,j)$-separator of $V(G)$.
Therefore, $\kappa(R_i^{\ell+1}) \le \kappa(R_i^{\ell})-1.$
Because I3 holds at $\ell$, $\kappa(R_i^{\ell+1}) \le \kappa(R_i^{1})-(\ell-1)-1 = \kappa(R_i^{1})-((\ell+1)-1)$.
So, I3 holds at iteration $\ell+1$, too.

Finally, we bound the number of queries.
First, by Lemma~\ref{lem:findseed}, \FindNewSeed\ makes at most $k \, \PackNum^*(\frac{\beta\gamma}{2+\gamma})$ \scq\ in total across all iterations.
Second, \FindCutEdge\ makes $\scO(\log n)$ queries at each iteration, see Observation~\ref{obs:bin_search}.
Third, \Cut\ makes at most $\PackNum^*(\beta\gamma)$ queries at each iteration, see Lemma~\ref{lem:algosep}.
Summing the three terms we obtain the claimed bound.
\end{proof}

\subsection{Proof of Lemma~\ref{lem:findseed}}
\findseedlemma*
We need to prove two technical results, Lemma~\ref{lem:cross_edge} and Lemma~\ref{lem:border}.
Then, we will prove Lemma~\ref{lem:findseed}.
\begin{lemma}
\label{lem:cross_edge}
Let $G$ be such that $C_i \subseteq V(G)$, and let $(S_i,S_j)$ be an $(i,j)$-separator for $V(G)$ obtained from $\Cut(V(G),u_i,u_j)$.
If $(x,y) \in \edgecut(S_i)$ and $d_G(u_i,x) \ge \frac{2}{\gamma}+1$, then $x \notin C_i$.
\end{lemma}
\begin{proof}
We show that $x \in C_i$ violates the geodesic convexity of $C_i$.
First $(x,y) \in \edgecut(S_i)$ implies $y \in S_j$.
Moreover,
$
    d_G(u_i,y) \ge d_G(u_i,x)-1 > \frac{1}{\gamma}
$.
Now recall the code of \Cut.
Since $d_G(u_i,y) > \frac{1}{\gamma}$, then $y \notin Z_j$.
But $y \in S_j = Z_j \cup U_j$, and therefore $y \in U_j$.
This means that \Cut\ executed line~\ref{line:cut_uj}, which happens only if:
\begin{align}
d_G(u_j,y) < d_G(u_i,y) \label{eq:d_uj_ui}
\end{align}
Now consider the path $\pi = (u_i,u_j)+\pi(u_j,y)+(y,x)$ from $u_i$ to $x$ where $\pi(u_j,y)$ has length $d_G(u_j,y)$.
First, we observe that $\pi$ is a simple path.
Suppose indeed by contradiction that $\pi$ is not simple.
Since $\pi(u_j,y)$ is simple (it is a shortest path), and since $u_i \ne x$ (because $d_G(u_i,x) \ge 1$ by assumption), we must have $u_i \in \pi(u_j,y)$ or $x \in \pi(u_j,y)$.
If $u_i \in \pi(u_j,y)$, then $d_G(u_j,y) > d_G(u_i, y)$, which contradicts~\eqref{eq:d_uj_ui}.
If instead $x \in \pi(u_j,y)$, then $d_G(u_j,y) > d_G(u_j,x)$, which gives:
\begin{align}
d_G(u_j, y) &> d_G(u_j, x)
\\ &\ge d_G(u_i, x) && \text{since } x \in S_i
\\ &\ge d_G(u_i, y) - 1 && \text{since } (x,y) \in E(G)
\end{align}
which, since $d_G$ is integral, implies $d_G(u_j, y) \ge d_G(u_i, y)$.
This contradicts again~\eqref{eq:d_uj_ui}.
Thus, $\pi$ is a simple path.

Now we show that $|\pi| \le d_G(u_i,x) \, (1+\gamma)$.
From~\eqref{eq:d_uj_ui}, we have:
\begin{align}
d_G(u_j,y) &\le d_G(u_i,y) - 1 && \text{since } d_G \in \Nat
\\ &\le d_G(u_i,x) + d_G(x,y) - 1 && \text{since } (x,y) \in E(G)
\\ &= d_G(u_i,x)
\end{align}
And therefore:
\begin{align}
|\pi| &= d_G(u_j,y) + 2
\\ &\le d_G(u_i,x) + 2 && \text{since } d_G(u_j,y) \le d_G(u_i,x)
\\ &\le d_G(u_i,x) \, (1+\gamma) && \text{since } d_G(u_i,x) \ge \frac{2}{\gamma}
\end{align}
Therefore $\pi$ is a simple path between two nodes of $C_i$ that violates the geodesic convexity of $C_i$.
So $x \notin C_i$, as claimed.
\end{proof}
Now recall \RecoverSingleCluster$(G,\epsilon,\bs,i)$.
Let $R_i^{\ell}$ be the value of $R_i$ at the beginning of the $\ell$-th iteration, and let $(S_i^{\ell},S_{j_{\ell}})$ be the separator computed by \Cut\ at the $\ell$-th iteration (note: the first cluster is always $i$, but the second cluster varies with $\ell$).
\begin{lemma}
\label{lem:border}
If $(x,y) \in \edgecut(R_i^{\ell})$, then $(x,y) \in \edgecut(S_i^{\tau})$ for some $\tau \in \{1,\ldots,\ell-1\}$.
\end{lemma}
\begin{proof}
Let:
\begin{align}
    \tau = \min\theset{1 \leq t \leq \ell-1}{(x,y) \in \edgecut(R_i^{t+1})}
\end{align}
First, we have $x \in S_i^{\tau}$. Indeed, by construction $R_i^{\tau+1} \subseteq R_i^{\tau} \cap S_i^{\tau}$, and we know $x \in R_i^{\tau+1}$. Second, we have $y \notin S_i^{\tau}$. Suppose indeed by contradiction that $y \in S_i^{\tau}$.
Since $x \in R_i^{\tau}$, and since $(x,y) \notin \edgecut(R_i^{\tau})$, then $y \in R_i^{\tau}$.
Therefore, $y \in S_i^{\tau} \cap R_i^{\tau}$. So $y$ would be connected to $x$ in $G[S_i^{\tau} \cap R_i^{\tau}]$, and therefore we would have $y \in R_i^{\tau+1}$ as well, by construction of $R_i^{\tau+1}$ as a connected component. But then, $y \in R_i^{\tau+1}$ would imply $(x,y) \notin \edgecut(R_i^{\tau+1})$ which contradicts our hypothesis. Therefore $x \in S_i^{\tau}$ and $y \notin S_i^{\tau}$, which implies $(x,y) \in \edgecut(S_i^{\tau})$, as claimed.
\end{proof}
\begin{proof}[\textbf{of Lemma~\ref{lem:findseed}}]
The first bound on the number of queries follows by Lemma~\ref{lem:ballsep}, by observing that $Z \subseteq B(u, \epsilon(\frac{2}{\gamma}+1)) = B(u, \epsilon\frac{2+\gamma}{\gamma})$.
Let us now prove the correctness; the claim on the adapted version will follow afterwards.

First, suppose that $R_i = C_i$.
Then $\bs \cap R_i = \{s_i\}$.
Moreover, at each iteration of the loop, by Lemma~\ref{lem:ballsep} $Z_i=Z$, so $Z \setminus Z_i = \emptyset$.
Finally, no edge $(x,y) \in \edgecut(R_i)$ exists such that $d_G(u,x) \ge \frac{2}{\gamma}+1$ for all $u \in \bu$.
Indeed, if such an edge $(x,y)$ existed, by Lemma~\ref{lem:border} we would have $(x,y) \in \edgecut(S_i)$ at some previous round of \RecoverSingleCluster, which by Lemma~\ref{lem:cross_edge} implies $x \notin C_i$ --- a contradiction, since $x \in R_i=C_i$.
Hence, \FindNewSeed\ reaches the last line and returns \nil.

Suppose now that $R_i \ne C_i$.
If $s_h \in R_i$ for some $s_h \ne s_i$, then \FindNewSeed\ returns $s_h$ at line~\ref{line:ret_sh}.
Otherwise, we must have $C_h \cap R_i \ne \emptyset$ but $C_h \nsubseteq R_i$ for some $h \ne i$.
By the connectedness of $C_h$ in $G$, this implies the existence of an edge $(x,y) \in \edgecut(R_i)$ with $x \in C_h$.
If any such edge exists with $d_G(u,x) < \frac{2}{\gamma}+1$ for some $u \in \bu$, then line~\ref{line:ret_ballsep} will find such an $x$ and return it.
Otherwise, any such edge has $d_G(u,x) \ge \frac{2}{\gamma}+1$ for all $u \in \bu$.
In this case, line~\ref{line:ret_cross} will return some $x$ such that $(x,y) \in \edgecut(R_i)$ and $d_G(u,x) \ge \frac{2}{\gamma}+1$ for all $u \in \bu$, which is correct since by Lemma~\ref{lem:border} and Lemma~\ref{lem:cross_edge} we have $x \notin C_i$.

To make \FindNewSeed\ use at most $k \, \PackNum^*(\frac{\beta\gamma}{2+\gamma})$ queries over the whole execution of \RecoverSingleCluster, we keep track of $Z \setminus Z_i$ in the following way.
At each invocation, we compute the set $Z^{(u)} = \{x \in R_i \,:\, d_G(u,x) < \frac{2}{\gamma}+1\}$, where $u$ is the last node added to $\bu$.
Then we invoke \BallSeparator\ only on $Z^{(u)}$, obtaining $Z^{(u)}_i$, and we then add $Z^{(u)} \setminus Z^{(u)}_i$ to $Z \setminus Z_i$.
Finally, we remove from $Z \setminus Z_i$ all points not in $R_i$.
This sequence of operations costs $ \PackNum^*(\frac{\beta\gamma}{2+\gamma})$ \scq\ queries, and the resulting set $Z \setminus Z_i$ will be exactly the one computed by \FindNewSeed\ above.
Hence the behavior of the algorithm is unchanged, but the total number of queries is at most $k \, \PackNum^*(\frac{\beta\gamma}{2+\gamma})$.
\end{proof}

\section{Supplementary material for Section~\ref{sec:general}}
\subsection{Lemma~\ref{lem:mineps}}
\label{app:general}
\begin{lemma}
\label{lem:mineps}
Let $\clus$ be a $(\beta,\gamma)$-convex $k$-clustering of $X$ (Definition~\ref{def:dense2}), and for $i\in\{1,\ldots,k\}$ let:
\begin{align}
  \zeta_i &= \min\{ \zeta : \forall \, x,y \in C \,:\, d_{G_{X}(\zeta)}(x,y)<\infty\} \label{eq:zeta1}
\\ 
 \zeta_i^* &= \min\{ \zeta : \rho(G_{C_i}(\zeta))=1\} \label{eq:zeta2}
\end{align}
Then $\zeta_i=\zeta_i^*$, and all properties of Definition~\ref{def:dense2} hold when $\rad_i=\zeta_i=\zeta_i^*$.
\end{lemma}
\begin{proof}
First, observe that $\zeta_i \le \zeta_i^*$, since $\rho(G_{C_i}(\zeta))=1$ implies $\forall \, x,y \in C \,:\, d_{G_{X}(\zeta)}(x,y)<\infty$, and thus the minimum in~\eqref{eq:zeta1} is taken over a superset of that of~\eqref{eq:zeta2}.
Now, consider the graph $G_{X}(\zeta_i)$.
For any two nodes $x,y \in C_i$, since $\zeta_i \le \zeta_i^*$ and $d_{G_{X}(\zeta_i)}(x,y)<\infty$, the geodesic convexity implies that any shortest path in $G_{X}(\zeta_i)$ between $x$ and $y$ lies in $C_i$.
But this means that the subgraph induced by $C_i$ in $G_{X}(\zeta_i)$ is connected. By definition of $\zeta_i^*$ this implies that $\zeta_i^* \le \zeta_i$.
We conclude that $\zeta_i = \zeta_i^*$.

For the second claim, we consider each property in turn when $\rad_i=\zeta_i^*$.
The connectedness of $G_{C_i}(\zeta_i^*)$ holds by definition of $\zeta_i^*$.
Now let $\zeta$ be any value such that the three properties hold when $\rad_i=\zeta$.
Then $\zeta \ge \zeta_i^*$, because for $\rad_i < \zeta_i^*$ the connectedness fails, by definition of $\zeta_i^*$.
This implies that the local metric margin and the geodesic convexity with margin hold for $\rad_i=\zeta_i^*$, since they hold for $\rad_i = \zeta \ge \zeta_i^*$.
\end{proof}

\subsection{Pseudo-code of \ReduceRecover\ and proof of Theorem~\ref{th:diff_rad}}
\diffradtheorem*
To prove the theorem, we need a technical lemma.
It guarantees that, if we take the connected component of the cluster with smallest radius $\rad^*$ in $G_X(\rad^*)$, then the clustering induced by $\clus$ on that subgraph is $(\beta,\gamma)$-convex (Definition~\ref{def:dense}) with radius $\rad^*$.

\begin{lemma}
\label{lem:still_dense}
Let $\clus$ be a $(\beta,\gamma)$-convex $\kappa$-clustering of $X$ (Definition~\ref{def:dense2}).
Let $\rad^*$ be the smallest radius of any cluster of $\clus$, and let $G^*$ be the connected component of $G_X(\rad^*)$ that contains a cluster with radius $\rad^*$.
Finally, let $X^*=V(G^*)$, and let $\clus^* = \big(C_1 \cap X^*, \ldots, C_{\kappa} \cap X^*\big)$.
Then, $\clus^*$ is a $(\beta,\gamma)$-convex clustering of $X^*$ with radius $\rad^*$ (Definition~\ref{def:dense}).
\end{lemma}
\begin{proof}
We verify that the three properties of Definition~\ref{def:dense} hold with radius $\rad^*$.
This implies that $\rad^*$ is also the smallest such value, since the corresponding cluster becomes disconnected in $G^*(\rad)$ for any $\rad < \rad^*$ --- and thus $\rad^*$ is indeed the radius of the clustering.
We define $C_i^* = C_i \cap X^*$ for all $i \in [\kappa]$.
Note that the graph on which we verify the properties is $G_{X^*}(\rad^*)$, which is precisely $G^*$ by the maximality of $X^*$ as a connected component.
Hence, from now on we write $G^*$ for $G_{X^*}(\rad^*)$.

\vspace*{5pt}
\noindent\emph{Connectivity:} the subgraph induced by $C_i^*$ in $G^*$ is connected.

\noindent\emph{Proof.} 
Consider two points $x,y \in C_i^*$; obviously $x,y \in C_i$. Since $G^*$ is connected, $d_{G^*}(x,y) < \infty$.
Moreover, $d_{G^*}(x,y) = d_{G_X(\rad^*)}(x,y)$, since by construction $G^*$ is the connected component of $G_X(\rad^*)$ containing $x$ and $y$.
Thus, $d_{G_X(\rad^*)}(x,y) < \infty$.
Moreover, $\rad^* \le \rad_i$ by assumption.
Thus, $x,y \in C_i$, and $d_{G_X(\rad^*)}(x,y) < \infty$ with $\rad^* \le \rad_i$.
Then, by the geodesic convexity of $\clus$ on $X$ (Definition~\ref{def:dense2}), any shortest path $\pi$ between $x$ and $y$ in $G_X(\rad^*)$ lies entirely in $C_i$.
Since again $G^*$ is the connected component of $G_X(\rad^*)$ containing $x,y$, then $\pi$ must lie in $X^*$.
We conclude that $\pi \in C_i \cap X^* = C_i^*$.
This holds for any choice of $x,y$.
Therefore, $G^*[C_i^*]$ is connected.

\vspace*{5pt}
\noindent\emph{Local metric margin:} for all $x,y \in X^*$, if $x \in C_i^*$ and $y \notin C_i^*$, then $d(x,y) > \beta\rad^*$.

\noindent\emph{Proof.} 
Since $x \in C_i^*$ and $y \notin C_i^*$, then $x \in C_i$ and $y \notin C_i$.
By the local metric margin of $\clus$, and since $\rad_i \ge \rad^*$, we have $d(x,y) > \beta \rad_i \ge \beta \rad^*$.

\vspace*{5pt}
\noindent\emph{Geodesic convexity with margin:} if $x,y \in C_i^*$, then in $G^*$ any simple path between $x$ and $y$ of length at most $(1+\gamma) d_{G^*}(x,y)$ lies entirely in $C_i^*$.

\noindent\emph{Proof.}
By the same argument of connectivity, we invoke the geodesic convexity (Definition~\ref{def:dense2}) for $x,y \in C_i$, for $\epsilon = \rad^* \le \rad_i$.
We obtain that $G_X(\rad^*)$ contains no simple path of length at most $(1+\gamma)d_{G_X(\rad^*)}(x,y)$ between $x$ and $y$ that leaves $C_i$.
This implies that no such path exists in $G^*$ as well, since $G^* \subseteq G_X(\rad^*)$.
Moreover, since $C_i^*=C_i \cap X^*$, no such path exists in $G^*$ that leaves $C_i^*$ (otherwise it would leave $C_i$).
Recalling that $(1+\gamma)d_{G_X(\rad^*)}(x,y)=(1+\gamma)d_{G^*}(x,y)$, we deduce that in $G^*$ there is no path of length at most $(1+\gamma)d_{G^*}(x,y)$ between $x$ and $y$ that leaves $C_i^*$.
This is the geodesic convexity of $C_i^*$ in $G^*$ (Definition~\ref{def:dense}).
\end{proof}
We can now present the algorithm for recovering $(\beta,\gamma)$-convex clusters with different radii.
\begin{algorithm2e}[ht!]
\caption{\ReduceRecover($X,\beps,\bs$)}
\label{algo:reduce_recover}
assume $\epsilon_1 \le \ldots \le \epsilon_k$\;
\textbf{for} {$i=1,\ldots,k$} \textbf{do}\; {\Indp
  $G^* :=$ the connected component of $s_i$ in $G_{X}(\epsilon_i)$, \, $X^* := V(G^*)$ \label{line:Gprime}\;
  \lFor{$j=i+1,\ldots,k$}{$s^*_j := \seed(X^*, j)$}
  $\bs^* := (s_i,s^*_{i+1},\ldots,s^*_k)$\;
  $\hat{C_i} := \RecoverSingleCluster(G^*,\rad_i,\bs^*,i)$\;
  output $\hat{C_i}$\;
  $X := X \setminus \hat{C_i}$\; 
\Indm}
\end{algorithm2e}

\begin{proof}[\textbf{of Theorem~\ref{th:diff_rad}}]
For each $i=1,\ldots,k$ let $X_i$ be the value of $X$ at the beginning of the $i$-th iteration of \ReduceRecover$(X,\beps,\bs)$.
We show that the following three invariants holds:
\begin{enumerate}\itemsep2pt
\item $X_i = C_i \cup \ldots \cup C_k$
\item $(C_i, \ldots, C_k)$ is a $(\beta,\gamma)$-convex clustering of $X_i$ (Definition~\ref{def:dense2})
\item if $i > 1$ then the algorithm has output $C_1,\ldots,C_{i-1}$ so far
\end{enumerate}
When $i=1$ we have $X_i=X$, and all invariants clearly hold.
Now assume that they hold at the beginning of the $i$-th iteration for some $i \ge 1$.
We will show that the algorithm sets $\hat{C_i}=C_i$.
This will imply that the three invariants hold at iteration $i+1$ as well.
For the first and third invariant, this is trivial.
For the second, simply observe that deleting a cluster never invalidates the three properties of Definition~\ref{def:dense2}; thus, $(C_{i+1}, \ldots, C_k)$ will be a $(\beta,\gamma)$-convex clustering of $X_{i+1} = C_{i+1} \cup \ldots \cup C_k$.

Thus, we prove that $\hat{C_i}=C_i$.
To this end, consider the subgraph $G^*$ and its node set $V^*$ computed at line~\ref{line:Gprime}.
Let $\clus_i=(C_i, \ldots, C_k)$, and let $\clus_i^*=(C_i \cap X^*, \ldots, C_k \cap X^*)$.
Since $\rad_i$ is the smallest radius of all clusters in $\clus_i$, then by Lemma~\ref{lem:still_dense} with $\rad^*=\rad_i$, $\clus_i^*$ is a $(\beta,\gamma)$-convex clustering for $X_i$ with radius $\rad_i$ (Definition~\ref{def:dense}).
Furthermore, by construction $\bs^*$ contains one seed for each nonempty cluster in $\clus^*$.
Therefore, by Lemma~\ref{lem:single_recovery}, $\RecoverSingleCluster(G^*,\epsilon,\bs^*,i)$ returns $C_i$, so $\hat{C_i}=C_i$.

The bound on the number of queries is straightforward.
\end{proof}

\section{Supplementary material for Section~\ref{sec:learn_params}}
\label{app:learn_params}
\subsection{\GetEpsilons\ and proof of Theorem~\ref{thm:epsilons}}
\epsitheorem*
We start with a simple routine for testing the connectedness of a cluster using \seed\ queries. 

\begin{algorithm2e}[ht!]
\label{algo:isconn}
\caption{\IsConnected($G,i$)}
 $u := \seed(V(G),i)$  \;
 $U := $ the connected component of $u$ in $G$\;
 \Return $(\seed(V(G) \setminus U,i) = \nil)$
\end{algorithm2e}
\begin{claim}
\label{claim:isconnected}
If $V(G) \cap C_i \ne \emptyset$, then \IsConnected$(G,i)$ uses two \seed\ queries and returns \textsc{true} if and only if $d_G(x,y) < \infty$ for all $x,y \in C_i$.
\end{claim}
Let $T$ be a minimum spanning tree of the weighted graph $\scG$.
For any $\epsilon > 0$, let $T(\epsilon)$ be the forest obtained by keeping only the edges $(x,y)$ of $T$ such that $d(x,y) \le \epsilon$.
Recall the following basic fact: 
\begin{claim}
\label{claim:mst}
The connected components of $T(\epsilon)$ are the connected components of $G_X(\epsilon)$.
\end{claim}
As a consequence, we have:
\begin{claim}
$\IsConnected(T(\epsilon), i)=\IsConnected(G_X(\epsilon), i)$, for any $i \in [k]$ and any $\epsilon > 0$.
\end{claim}
We introduce the algorithm for learning the radius of a single cluster.
\begin{algorithm2e}
\label{algo:epsilon}
\caption{\GetEpsilon($T, i$)}
$\pmb{w}:=(w_0,w_1,\ldots,w_{\ell})$, the distinct edge weights of $T$ in increasing order, with $w_0=0$\;
$lo := 0$, \quad $hi := \ell$\;
\textbf{while} $w_{lo} < w_{hi}$ \textbf{do}\;
  \Indp
  $mid := \lfloor \frac{lo+hi}{2} \rfloor$\;
  \textbf{if} {\IsConnected$(T(w_{mid}), i)$} \label{line:test_neps} \textbf{then} $hi:=mid$ \textbf{else} $lo:=mid+1$\;
  \Indm
\Return $w_{hi}$\;
\end{algorithm2e}
\begin{lemma}
If $T$ is a MST of $\scG=(X,\scE,d)$, then \GetEpsilon$(T, i)$ returns $\rad_i$ in time $\scO(n \log n)$ using $\scO(\log n)$ \seed\ queries.
\end{lemma}
\begin{proof}
It is straightforward to see that the algorithm stops within $\scO(\log n)$ iterations, since $\pmb{w}$ has at most $m = \scO(n^2)$ entries and $(hi-lo)$ decreases by a constant factor at each iteration.
For the running time, since $T$ has $\scO(n)$ edges, every call to \IsConnected$(T(w_{mid}), i)$ takes time $\scO(n)$.
This gives the time bound of $\scO(n \log n)$.

Now we show that the algorithm returns $\rad_i$.
By Lemma~\ref{lem:mineps}, this is equivalent to prove that the algorithm returns $w^* = \min\{ w \in \pmb{w} : C_i \text{ is connected in } G_{X}(w) \}$.
Consider the beginning of a generic iteration, when the test $w_{lo} < w_{hi}$ is performed.
We claim that $w^* \le w_{hi}$.
To this end, observe that $C_i$ is connected in $G_X(w_{hi})$.
This is true since it holds at the beginning of the first iteration, when $w_{hi} = w_{\ell}$, and because at each iteration $hi$ is set to $mid$ only if \IsConnected$(T(w_{mid}), i)=\textsc{true}$.
We now claim that $w^* \ge w_{lo}$.
This holds since $w^* \ge w_{0} = 0$ at the first iteration, and because at each iteration $lo$ is set to $mid+1$ only if \IsConnected$(T(w_{mid}), i)=\textsc{false}$.
Therefore, when the algorithm stops, we have $w_{lo}=w_{hi}=w^*$, as claimed.
\end{proof}
We conclude with the algorithm to learn all the radii.
We denote by $\MST(m)$ the time needed for computing the MST of a connected graph (note that we can always assume $\scG$ is connected, otherwise we can just compute its connected components in time $\scO(m)$ and use each one of them in turn).
It is known that $\MST(m)=O(m\, \alpha(m,m))$, where $\alpha(m,m)$ is the classic functional inverse of Ackermann's function~\citep{ChazelleMST}.
Lemma~\ref{lem:epsilons} below follows immediately from these observations.
\begin{algorithm2e}
\label{algo:epsilons}
\caption{\GetEpsilons($\scG=(X,\scE,d), k$)}
$T:= \MST(\scG)$\;
\textbf{for} {$i=1,\ldots,k$} \textbf{do}\;
  \Indp
  $\hat{\rad_i} := \GetEpsilon(T,i)$\;
\Indm
\Return $\hat{\rad_1},\ldots,\hat{\rad_k}$\;
\end{algorithm2e}
\begin{lemma}
\label{lem:epsilons}
\GetEpsilons$(\scG, k)$ returns the radii $\epsilon_1,\ldots,\epsilon_k$ in time $\scO(m \, \alpha(m,n) + k n \log n)$ using $\scO(k \log n)$ \seed\ queries.
\end{lemma}

\subsection{Proof of Theorem~\ref{lem:guessing}}
\guessingtheorem*
\begin{proof}
Suppose first $\beta$ is unknown and $\gamma$ is known.
Recall that $\beta \le 1$.
We make a succession of guesses $\hat{\beta} = 2^{-j}$ for $j=0,1,\dots$.
For each guess, we run our algorithm with $\beta = \hat{\beta}$ and look at the output clustering $\hat{\clus}$.
Clearly, if $\clus$ is $(\beta,\gamma)$-convex, then $\clus$ is $(\hat{\beta},\gamma)$-convex for any $\hat{\beta} \le \beta$ as well.
Thus, as soon as $\hat{\beta} \le \beta$, our algorithm will return $\hat{\clus} = \clus$.
So, after each run, we need only to check whether $\hat{\clus} = \clus$, and stop in the affirmative case.

To check whether $\hat{\clus} = \clus$, we do as follows.
First, we check if $|\hat{\clus}| \ne |\clus|$.
If this is the case, then the only possibility for $\hat{\clus} \ne \clus$ is that some cluster $C_i$ intersects both $\hat{C}$ and $X \setminus \hat{C}$, for some cluster $\hat{C} \in \hat{\clus}$.
Therefore, we take each cluster $\hat{C} \in \hat{\clus}$ in turn.
We then take any node $x \in \hat{C}$, and we learn the label of $i$ with $\scO(k)$ \scq\ queries.
Then, we invoke $\seed(X \setminus \hat{C}, i)$.
If we get a node in return, we know that $C_i$ has points in $\hat{C}$ and $X \setminus \hat{C}$, and therefore $\hat{\clus} \ne \clus$.
Otherwise, we continue to the next cluster.
If the outputs of the $\seed$ are all $\nil$, then we deduce that $\hat{\clus} = \clus$.

The process will stop with $\hat{\beta} \ge \frac{1}{2} \beta$, which happens after $R=\scO(\log \PackNum^*(\frac{\beta\gamma}{2}))$ rounds.
Since $\PackNum^*(\frac{\beta\gamma}{2}) \le \big(\frac{4}{\beta\gamma}\big)^{\dens(X)}$, see Section~\ref{sec:prelim}, then $\log \PackNum^*\big(\frac{\beta\gamma}{2}\big) = \scO(\dens(X)\log\big(\frac{4}{\beta\gamma}\big))$.
At each round, the algorithm uses $k^2$ \scq\ queries plus $2k$ \seed\ queries.
Thus, in total we use $\scO(k^2 R)$ \scq\ queries and $\scO(k R)$ \seed\ queries.
The case with $\gamma$ is unknown and $\beta$ is known is symmetric.
\end{proof}

\section{Supplementary material for Section~\ref{sec:runtime}}
\label{app:runtime}

\subsection{Running time with identical radii}
We prove:
\begin{theorem}
\label{thm:same_eps_runtime}
\RecoverClustering$(X,\epsilon,\bs)$ runs in time $\scO(k^2(n+m))$.
\end{theorem}
\begin{proof}
First, recall from Section~\ref{sec:runtime} that computing $G=G_X(\epsilon)$ takes time $\scO(n+m)$.
Then, each call to \RecoverSingleCluster$(G,\epsilon,\bs,i)$ takes time  $\scO(k(n+m))$ by Lemma~\ref{lem:single_runtime}.
Since there are $k$ clusters, the claim follows.
\end{proof}

In the rest of this appendix we prove Lemma~\ref{lem:single_runtime}, through a sequence of intermediate steps.
\begin{lemma}
\label{lem:abs_runtime}
\BallSeparator$(Z,\epsilon,u_i)$ runs in time $\scO(n+m)$.
\end{lemma}
\begin{proof}
First, we construct $G(\beta\epsilon)$ by thresholding $G$, which takes time $\scO(n+m)$.
Then, we keep only the edges of $G(\beta\epsilon)$ which have both endpoints in $Z$, which takes again time $\scO(n+m)$.
Once we have $G_Z(\beta\epsilon)$, listing its connected components takes once again time $\scO(n+m)$.
\end{proof}

\begin{lemma}
\label{lem:bsearch_runtime}
Given a simple $C_i$-prefixed path $\pi$, \FindCutEdge$(\pi)$ runs in time $\scO(\log n)$.
\end{lemma}
\begin{proof}
Straightforward, see Observation~\ref{obs:bin_search} and the code of \FindCutEdge.
\end{proof}

\begin{lemma}
\label{lem:cut_runtime}
\Cut$(G,u_i,u_j)$ runs in time $\scO(n+m)$.
\end{lemma}
\begin{proof}
Computing $d_G(u_i,x)$ and $d_G(u_j,x)$ for all $x \in G$ takes time $\scO(n+m)$ using a BFS from $u_i$ and $u_j$.
Thereafter, we can compute $Z$ in time $\scO(n)$.
Running \BallSeparator($Z,\epsilon,u_i$) takes time $\scO(n+m)$, see above.
Finally, the loop at line~\ref{line:cycle} takes time $\scO(n)$.
\end{proof}

\begin{lemma}
\label{lem:findseed_runtime}
In \RecoverSingleCluster$(G,\epsilon,\bs,i)$, each call to \FindNewSeed$(G,R_i,\epsilon,\bs,\bu,i)$ takes time $\scO(k(n+m))$.
By adapting both algorithms, this can be reduced to $\scO(n+m)$ while adding at most an additive $\scO(n+m)$ to the running time of each iteration of \RecoverSingleCluster.
\end{lemma}
\begin{proof}
Let us start with the $\scO(k(n+m))$ bound given by a ``naive'' implementation of \FindNewSeed.
At line~\ref{line:if1}, we compute $\bs \cap R_i$ in time $\scO(k|R_i|) = \scO(k n)$ and perform the check in constant time.
At line~\ref{line:for}, we make $|\bu| \le k$ iterations.
At each iteration we compute $Z$ in time $\scO(n+m)$ with a BFS from $u$, then we run \BallSeparator($Z,\epsilon,u$) in time $\scO(n+m)$ by Lemma~\ref{lem:abs_runtime}, and possibly we search for $x \in Z \setminus Z_i$ which takes time $\scO(n)$.
Hence the entire loop of line~\ref{line:for} takes time $\scO(k(n+m))$.
Finally, at line~\ref{line:if2} we compute the set $\{(x,y) \in \edgecut(R_i) \,:\, \forall u \in \bu \,:\, d_G(u,x) \ge \frac{2}{\gamma}+1\}$.
To this end, we compute the set $\{x \in R_i \,:\, \forall u \in \bu \,:\, d_G(u,x) \ge \frac{2}{\gamma}+1\}$, which takes time $\scO(k (n+m))$ using a BFS from $u$.
For each such $x$ in this set, we list all its edges $(x,y) \in E(G)$.
If we find any such edge with $y \notin R_i$, we return $y$, else we return \nil.
Thus, this part takes $\scO(k (n+m))$.
Therefore, a single call to \FindNewSeed\ takes time $\scO(k(n+m))$.
Since \FindNewSeed\ is called at most $k$ times, this gives a total running time of $\scO(k^2(n+m))$.

Let us now see how to reduce to $\scO(n+m)$ the running time of \FindNewSeed, by adding at most $\scO(n+m)$ to each iteration of \RecoverSingleCluster.
First, consider line~\ref{line:if1} of \FindNewSeed.
We keep $\bs$ updated so as to ensure that $\bs \cap R_i = \bs \setminus \{s_i\}$.
In this way, we can run line~\ref{line:if1} in constant time.
Towards this end, we modify \RecoverSingleCluster\ as follows.
First, we store $s_i$ separately form $\bs$ in a dedicated variable.
Second, after updating $G_i$ and $R_i$ at line~\ref{line:connect-1} of \RecoverSingleCluster, we replace $\bs$ with $\bs \cap R_i$.
This is done by taking an empty dictionary $\bs'$, taking every node $x \in R_i$, and adding $x$ to $\bs'$ if $x \in \bs$.
The whole operation takes time $\scO(n)$ by using dictionaries with $O(1)$ time per lookup and update.
Summarizing, we spend an additional $\scO(n)$ time at each iteration of \RecoverSingleCluster, and line~\ref{line:if1} of \FindNewSeed\ will run in constant time.

Now consider the loop at line~\ref{line:for} of \FindNewSeed.
We modify \RecoverSingleCluster\ so as to keep track of the set of nodes:
\[
  \overline{Z(\bu)} = \{x \in R_i \setminus C_i \,:\, \exists u \in \bu \,:\, d_G(u,x) < \frac{2}{\gamma}+1\}
\]
Note that line~\ref{line:for} of \FindNewSeed\ detects precisely if $\overline{Z(\bu)} \ne \emptyset$, in which case it returns any $x \in \overline{Z(\bu)}$.
Thus, if we have $\overline{Z(\bu)}$, we can replace the entire block at line~\ref{line:for} with an equivalent block that runs in time $\scO(1)$.
To keep track of $\overline{Z(\bu)}$, we initialize it to an empty set, using a dictionary with $O(1)$ lookup and access time.
Then, after updating $G_i$ and $R_i$ at line~\ref{line:connect-1}, we perform the following operations.
First, we make \RecoverSingleCluster\ compute $Z=Z(u_i)$.
Second, we run \BallSeparator($Z,\epsilon,u_i$) to obtain $Z_i(u_i)$.
Third, we compute $\overline{Z_i(u_i)} := Z(u_i) \setminus Z_i(u_i)$.
Fourth, we add $\overline{Z_i(u_i)}$ to $\overline{Z(\bu)}$.
Finally, we keep in $\overline{Z(\bu)}$ only those $x \in \overline{Z(\bu)}$ such that $x \in R_i$.
This can be done by creating a new dictionary, adding to it each $x \in R_i$ such that $x \in \overline{Z(\bu)}$, and overwriting $\overline{Z(\bu)}$ with that dictionary, which requires time $\scO(n)$ in total.
Therefore, we add $\scO(n+m)$ to each iteration of \RecoverSingleCluster, and line~\ref{line:for} of \FindNewSeed\ will take time $\scO(n)$.

Finally, we have line~\ref{line:if2} of \FindNewSeed.
Here, we want to perform the check in time $\scO(n+m)$.
To this end, at the beginning of the first iteration of \RecoverSingleCluster, we mark all nodes of $R_i$ as \emph{active}.
Then, at each iteration, after \RecoverSingleCluster\ has computed $(u_i,u_j)$, for all $x \in R_i$ we compute $d_G(u_i,x)$ and if $d_G(u_i,x) < \frac{2}{\gamma}+1$ then we change the mark of $x$ to \emph{inactive}.
This takes time $\scO(n+m)$, using a BFS from $u_i$.
Then, at line~\ref{line:if2} of \FindNewSeed, we only need to sweep over all $x \in R_i$ and, if $x$ is active, list its edges $(x,y)$ until finding $y \notin R_i$ (a check which takes again time $O(1)$ by storing $R_i$ as a dictionary).
This gives a total time bound of $\scO(n+m)$ for the block at line~\ref{line:if2}, and once again we add only a $\scO(n+m)$ to each iteration of \RecoverSingleCluster.

The proof is complete.
\end{proof}

\begin{lemma}
\label{lem:single_runtime}
\RecoverSingleCluster$(G,\epsilon,\bs,i)$ runs in time $\scO(k(n+m))$.
\end{lemma}
\begin{proof}
Computing $G_i$ and $R_i$ at any point along the algorithm takes time $\scO(n+m)$.
Now consider each iteration of the main loop.
By Lemma~\ref{lem:findseed_runtime}, we can make \FindNewSeed$(G,R_i,\epsilon,\bs,\bu,i)$ run in time $\scO(n+m)$ while increasing the overall running time of the iteration of \RecoverSingleCluster\ by $\scO(n+m)$.
ShortestPath$(G[R_i],s_i,s_h)$ runs in time $\scO(n+m)$, as it is simply a BFS on $G[R_i]$.
\FindCutEdge$(\pi(s_i,s_h))$ runs in time $\scO(\log n)$, see Observation~\ref{lem:bsearch_runtime}.
Adding $u_i$ to $\bu$ takes constant time.
\Cut$(G,u_i,u_j)$ runs in time $\scO(n+m)$, see Lemma~\ref{lem:cut_runtime}.
Therefore each iteration of \RecoverSingleCluster\ takes time $\scO(n+m)$.
At most $k$ iterations are made, concluding the proof.
\end{proof}

\subsection{Running time with different radii}
\begin{theorem}
\ReduceRecover$(X,\beps,\bs)$ runs in time $\scO(k^2(n+m))$.
\end{theorem}
\begin{proof}
Let us consider each one of the $k$ iterations of the algorithm.
To compute $G^*$, we perform a BFS by ignoring any edge $(x,y) \in \scG$ with  $d(x,y) > \rad^*$.
The resulting runtime is $\scO(n+m)$.
The \seed\ part takes time $\scO(k)=\scO(n)$, as does the construction of $\bs^*$.
The call to \RecoverSingleCluster\ takes time  $\scO(k(n+m))$ by Lemma~\ref{lem:single_runtime}.
Writing $\hat{C_i}$ in the output takes time $\scO(n)$.
Summing over all iterations gives the bound.
\end{proof}

\section{Supplementary material for Section~\ref{sec:lb}}
\label{app:lb}

\subsection{Proof of Theorem~\ref{th:pack_dep}}
\packdeptheorem*
\begin{proof}
Let $\scG=(X,\scE,d)$ where $(X,\scE)$ is the complete graph on $n$ nodes and $d=1$, and let $\clus=(C_1,C_2)$ be a uniform random partition of $X$.
We claim that any such $\clus$ is $(\beta,\gamma)$-convex according to both Definition~\ref{def:dense} and Definition~\ref{def:dense2}.
Take indeed $\epsilon=1$ and let $G=G_X(\epsilon)$.
The connectivity of $G[C_1]$ and $G[C_2]$ holds trivially (they are complete graphs).
The local metric margin holds as well, since any two distinct points $x,y \in X$ satisfy $d(x,y) = d > \beta d$, as $\beta < 1$.
To see that geodesic convexity holds, too, note that for any $x,y \in C_1$ we have $d_G(x,y) \le 1$ and any (simple) path between $x$ and $y$ that contains a point in $X \setminus C_1$ has length at least $2 > (1+\gamma) d_G(x,y)$.
Finally, note that $G_X(\epsilon')$ is an independent set for any $\epsilon' < \epsilon$, proving that $\clus$ is $(\beta,\gamma)$-convex according to Definition~\ref{def:dense2} as well.

Now, since $\clus$ is chosen uniformly at random among all partitions of $X$, one can see that $\Omega(n)$ \scq\ or \seed\ queries are necessary to recover $\clus$ with constant probability.
To see this, note that as long as $x$ has not been returned by some \seed\ query or has not bee queried via \scq, then $x$ belongs to one of $C_1$ and $C_2$ with equal probability.
Finally, $\dens(X) = \log_2 \mu(X) \le \log_2 n$ since $\mu(X) \le |X|$.
Thus $n=2^{\dens(X)}$, which proves the thesis.
\end{proof}

\subsection{Proof of Theorem~\ref{th:nec_seeds}}
\necseedstheorem*
\begin{proof}
Let $\scG=(X,\scE,d)$ where $(X,\scE)$ is the complete graph and $d(x,y)$ is the Euclidean distance in $\R^2$.
We let $X = \text{UP} \cup \text{LOW}$, see Figure~\ref{fig:caterpillar}, where:
\[
\text{UP} = \bigcup_{j=1}^{n/3} \{(2j,1)\}, \quad
\text{LOW} = \bigcup_{j=1}^{2n/3}\{(j,0)\}
\]
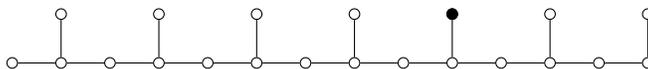
\begin{figure}[h!]
\centering
\begin{tikzpicture}[scale=.65,every node/.style={draw,fill,circle,minimum size=4pt,inner sep=0pt}]
\def\m{7}
\pgfmathtruncatemacro\n{2*\m}
\foreach \i in {1,...,\n} {
    \node[fill=white] (x\i) at (\i,0) {};
}
\pgfmathtruncatemacro\nend{\n-1}
\foreach \i in {1,...,\nend} {
    \pgfmathtruncatemacro\j{\i+1}
    \path[-] (x\i) edge (x\j);
}
\foreach \i in {1,...,\m} {
    \node[fill=white] (y\i) at (\i*2,1) {};
}
\node[fill=black] (z) at (y5) {};
\pgfmathtruncatemacro\mend{\m}
\foreach \i in {1,...,\mend} {
    \pgfmathtruncatemacro\j{\i*2}
    \path[-] (y\i) edge (x\j);
}
\end{tikzpicture}
    \caption{the graph $G_{\epsilon}(X)$ for $\epsilon=1$. All points are in $C_1$, save for the filled point in $C_2$, chosen uniformly at random in UP.}
    \label{fig:caterpillar}
\end{figure}

Now choose a point $z$ uniformly at random from UP, and set $C_1 = X \setminus \{z\}$ and $C_2 = \{z\}$.
One can check that all properties of Definition~\ref{def:dense} and Definition~\ref{def:dense2} are satisfied for $\epsilon=1$.
In particular, in $G_X(\epsilon)$, no simple path between two points of $C_1$ contains $z$.
Hence, $\clus$ is $(\beta,\gamma)$-convex.
Clearly, $\Omega(n)$ queries are needed to find $z$ (which is equivalent to recovering $\clus$) with constant probability, even if $\gamma,\alpha$ and the $\epsilon$ are known.
\end{proof}

\subsection{Proof of Theorem~\ref{thm:learn_radii_lb}}
\radiilowertheorem*
\begin{proof}
We start with the simpler case $k=2$.
Let $\scG=(X,\scE,d)$ be a path with increasing edge weights, that is:
\begin{align}
    X & =[n]
    \\
    \scE &= \left\{(j,j+1) \,:\, j=1,\ldots,n-1\right\}
    \\
    d(j,j+1) &=1 + \frac{\beta j}{n}, \quad (j,j+1) \in \scE
\end{align}
Choose $j^*$ uniformly at random in $\{2,\ldots,n-1\}$.
We let $C_1=\{1,\ldots,j^*\}$, and \mbox{$C_2=\{j^*\!+\!1,\dots,n\}$}.

First, we prove that $C_1$ and $C_2$ are $(\beta,\gamma)$-convex with radii respectively $\epsilon_1=d(j^*-1,j^*)$ and $\epsilon_2=d(n-1,n)$.
Recall Definition~\ref{def:dense}.
For the connectivity, clearly $\rho(G_{C_1}(\epsilon_1))=\rho(G_{C_2}(\epsilon_2))=1$.
For the local metric margin, note that, by the choice of $\beta$ and $d$, we have $d \ge 1$ but $\epsilon_1,\epsilon_2 \le \frac{3}{2}$.
Thus, for any two distinct $x,y \in X$, we have $d(x,y) > \beta \max(\epsilon_1,\epsilon_2)$.
Therefore the local metric margin is satisfied.
For geodesic convexity, note that there is only one edge between $C_1$ and $C_2$ in $\scG$, thus no simple path can exist between two points of one cluster that intersects the other cluster.
Thus, the properties of Definition~\ref{def:dense} are satisfied by $\epsilon_1,\epsilon_2$.

To show that Definition~\ref{def:dense2} is satisfied as well, we have to prove that $\epsilon_1,\epsilon_2$ are the smallest such values; by Lemma~\ref{lem:mineps} this implies that $\epsilon_1$ and $\epsilon_2$ are the radii of respectively $C_1$ and $C_2$.
To this end, simply note that $\epsilon_1 = \min\{ \zeta : \rho(G_{C_1}(\zeta))=1\}$, since $d(j^*-1,j^*)=\epsilon_1$ and $d(j,j+1) \le \epsilon_1$ for all $j=1,\ldots,j^*-1$.
Similarly, $\epsilon_2 = \min\{ \zeta : \rho(G_{C_2}(\zeta))=1\}$, since $d(n-1,n)=\epsilon_2$ and $d(j,j+1) \le \epsilon_1$ for all $j=j^*,\ldots,n-1$.

Finally, we prove that any algorithm needs $\Omega(\log n)$ queries to learn $\epsilon_1$.
Clearly, if the algorithm learns $\epsilon_1$ then it can also output the index $j^*$, which is a function of $\epsilon_1$.
Therefore, we show that finding $j^*$ requires $\Omega(\log n)$ queries.

First, we show that \seed\ is as powerful as \scq.
Consider any set of points $U \subseteq X$.
Recall that, when $U \cap C_i \ne \emptyset$, \seed$(U,i)$ is allowed to return \emph{any} node in $U \cap C_i$.
Therefore, we let \seed$(U,i)$ return $\min(U \cap C_i)$ if $i=1$, and $\max (U \cap C_i)$ if $i=2$.
Now, observe that $\min(U \cap C_1) = \min U$ and $\max(U \cap C_1) = \max U$.
Therefore, the output of \seed$(U,i)$ can be emulated using \scq.
If $i=1$, then we run \scq$(\min(U),1)$; if the response is $+1$, then we return $\min(U)$, else we return \nil.
For $i=2$, we do the same, but using $\max(U)$.

Therefore, \seed\ and \scq\ are equivalent in this case.
Since each call to \scq\ reveals at most one bit of information, and $j^*$ is chosen uniformly at random in a set of cardinality $\Omega(n)$, we need $\Omega(\log n)$ in order to learn $j^*$ with constant probability.

In order to extend the construction to any $k \ge 2$, simply take $K=\frac{k}{2}$ disjoint weighted paths on $\frac{n}{K}$ nodes each (without loss of generality we can assume $k$ is even).
Each such path is weighted as in the construction above, with the weights of the $h$-th path all smaller than the weights of the $(h+1)$-th path.
For each path, we draw $j^*$ uniformly at random like above, and form two clusters.
The same proof used above shows that, to learn the radii of all clusters with constant probability, any algorithm uses at least $\Omega(k \log \frac{n}{k})$ queries.
\end{proof}

\end{document}